\documentclass[journal]{IEEEtran}
\usepackage{amsmath,amssymb,mathtools,amsthm,bm,nicefrac}
\usepackage{graphicx,booktabs,multirow,xcolor}
\usepackage{tikz}
\usepackage{algorithm,algorithmic}
\usepackage{enumitem}
\usepackage{refcount}
\usepackage[numbers,sort&compress]{natbib}

\definecolor{linkblue}{rgb}{0.21,0.49,0.74}
\usepackage[breaklinks,colorlinks,allcolors=linkblue]{hyperref}
\usepackage[capitalize]{cleveref}

\theoremstyle{plain}
\newtheorem{theorem}{Theorem}

\AddToHook{env/proposition/begin}{\crefalias{theorem}{proposition}}
\AddToHook{env/lemma/begin}{\crefalias{theorem}{lemma}}
\AddToHook{env/corollary/begin}{\crefalias{theorem}{corollary}}
\theoremstyle{definition}

\newtheorem{assumption}{Assumption}
\theoremstyle{remark}

\title{Protective Perturbations Must Survive the Resize: Scale-Robust Image Immunization against Malicious Editing}
\author{Zhongliang~Guo,
        Yan~Lin,
        Yifei~Qian,
        Ognjen Arandjelovi\'c,
        Daizong~Liu,
        Jianfeng Dong,\\
        and~Chun Pong Lau, \IEEEmembership{Member, IEEE}%
\thanks{Zhongliang Guo is with the School of Natural and Computing Science, University of Aberdeen, Aberdeen, U.K. (email: zhongliang.guo@abdn.ac.uk).}%
\thanks{Yan Lin is with the School of Computing, Newcastle University, Newcastle upon Tyne, U.K. (email: y.lin64@ncl.ac.uk).}%
\thanks{Yifei Qian is with the School of Computer Science, University of Nottingham, Nottingham, U.K. (email: yifei.qian1@nottingham.ac.uk).}%
\thanks{Ognjen Arandjelovi\'c is with the School of Computer Science, University of St Andrews, St Andrews, U.K. (email: oa7@st-andrews.ac.uk).}%
\thanks{Daizong Liu is with the Institute for Math \& AI, Wuhan University, Wuhan, China (email: daizongliu@whu.edu.cn).}%
\thanks{Jianfeng Dong is with the College of Computer and Information Engineering, Zhejiang Gongshang University, Hangzhou, China (email: jianfeng.dong@zjgsu.edu.cn).}%
\thanks{Chun Pong Lau is with the Department of Data Science, City University of Hong Kong, Hong Kong, China (email: cplau27@cityu.edu.hk).}
\thanks{Corresponding author(s): Daizong Liu, Jianfeng Dong and Chun Pong Lau.}
}
\begin{document}
\maketitle
\begin{figure*}[!t]
\centering
\includegraphics[width=\linewidth]{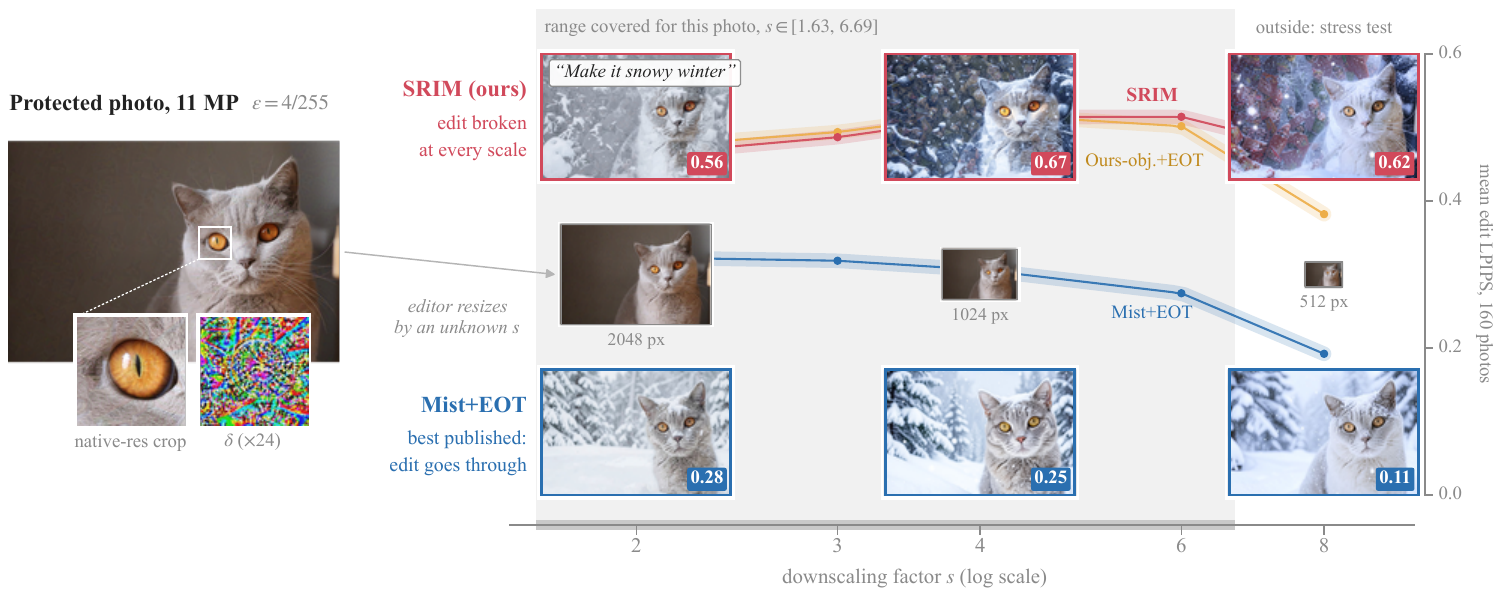}
\caption{A protected 11-megapixel photo edited by FLUX.2-klein (``Make it snowy winter'') after downscaling by $s$. Top and bottom rows: real edits under SRIM and under Mist+EOT, the strongest published protection with EOT; badges give the edit LPIPS to the unprotected edit (higher means a stronger disruption). Middle row: the editor's inputs, drawn to scale (long side 2048, 1024, and 512 pixels). Ribbons: mean edit LPIPS over the 160 test photos (right scale), including our objective with EOT, which drops at the large-scale end ($s=8$). Shaded: the scale range covered for this photo; $s=8$ lies outside it. Insets: a native-resolution crop of the protected photo and the perturbation amplified $\times$24.}
\label{fig:teaser}
\end{figure*}
\begin{abstract}
While instruction-guided editors built on diffusion transformers can turn a personal photo into a convincing forgery with a single sentence, protective perturbations aim to make such edits fail before the photo is ever shared.
The threat is most acute for the photos people actually share, which have 10 megapixels or more and which current editors first downscale by a factor that the owner does not know.
Existing protections are computed at one resolution or averaged over random scales, and they lose most of their effect once the photo is downscaled.
They implicitly assume that the editor sees the scale for which they were optimized, or that protecting the average scale protects every scale.
We propose SRIM, a protection that remains effective at every scale an editor may use, without knowing its preprocessing and at the same computational cost as standard expectation over transformation.
SRIM covers the range of scales implied by the photo size with anchor scales and spends each optimization step on the scale where the protection is currently weakest.
A spectral analysis shows that this worst-case protection degrades only logarithmically with the width of the range, a value that averaging over scales cannot reach.
Extensive experiments on full-resolution photos and popular instruction-guided editors validate that SRIM gives the strongest protection, outperforming existing methods even with a quarter of their $\ell_\infty$ budget, remaining the strongest under input purification.
Our code is available at \url{https://github.com/ZhongliangGuo/SRIM}.
\end{abstract}
\begin{IEEEkeywords}
Image immunization, adversarial perturbation, diffusion models, image editing, scale robustness.
\end{IEEEkeywords}
\section{Introduction}
\label{sec:intro}

\IEEEPARstart{W}{hile} instruction-guided editors built on diffusion and flow-matching transformers~\cite{peebles2023dit,lipman2023flowmatching}, such as FLUX.1 Kontext~\cite{labs2025kontext}, Qwen-Image-Edit~\cite{wu2025qwenimage}, and FLUX.2, have made photo editing accessible to everyone, they also let anyone turn a downloaded personal photo into a convincing forgery with a single sentence of text.
A growing line of work therefore asks the owner to \emph{immunize} a photo before sharing it: a small adversarial perturbation is added so that later edits of the image fail~\cite{salman2023photoguard,liang2023advdm,liang2023mist,xue2024sds,choi2025diffusionguard,wang2025facelock,guo2025pca,shen2025decontext,liu2026diffprotect,yam2025myface}.

In this paper, we focus on immunization in the setting in which it is used.
The owner shares the photo at its native resolution, typically 10 to 30 megapixels, while the editor resizes its input to between $0.25$ and $4$ megapixels before encoding it.
The editor therefore sees the protected photo only after an antialiased downscaling by a factor $s$ that the owner does not know.
A protection that fails after this resize leaves the owner unprotected exactly in the situation that immunization is meant to cover.

In general, existing protections can be roughly divided into two categories.
Protections of the first category are computed at one resolution: either the editor's working resolution, typically $512\times512$ pixels for Stable Diffusion editors and about one megapixel for diffusion transformers, or the native resolution of the photo.
Protections of the second category average the perturbation over random transformations with expectation over transformation (EOT)~\cite{athalye2018eot,yao2025eolt,huang2023erosion}.
Despite their effectiveness at the editor's working resolution, both categories lose most of their effect at the scales that real editors use: a single $0.5\times$ resize already removes most of the protection of existing methods~\cite{sun2026doprotective}.
As shown in \cref{fig:teaser}, once the photo is downscaled, the strongest published protection with EOT no longer stops the edit.
We argue that this limitation stems from a latent assumption of existing protections: that the editor sees the protected image at the scale for which the protection was optimized, or, with EOT, that protecting the average scale protects every scale.
Viewed as a frequency-selective channel, the resize discards most of a native-resolution perturbation, keeps a fixed-resolution perturbation effective only in a narrow window of scales, and leaves the ends of the range uncovered when the protection is averaged over scales (\cref{sec:theory}).

An ideal protection should therefore satisfy three requirements:
\begin{enumerate}[label=(\roman*)]
    \item It should remain effective at every scale in the range that an unknown editor may use, which fixed-resolution protections fail to meet because they cover only a window of scales.
    \item It should protect the weakest scale of this range rather than the average one, which EOT fails to meet because it trades the ends of the range for its middle.
    \item It should be computable on a 10 to 30 megapixel photo at the memory and time of existing protections, without knowing the editor's preprocessing, which native-resolution application fails to meet because it exceeds the memory of a 48\,GB GPU.
\end{enumerate}
Meeting (ii) without sacrificing (iii) is the central difficulty, since evaluating every scale at every step multiplies the cost by the number of scales.

To this end, we propose SRIM (Scale-Robust IMmunization), whose key idea is to treat the unknown resize as a worst case over a range of scales, in the spirit of group distributionally robust optimization~\cite{sagawa2020dro}, rather than as one more random transformation.
To meet (i), SRIM places anchor scales on a grid that covers the whole range implied by the photo size, with random jitter and resampling kernels.
To meet (ii), it draws one anchor per step from a distribution that tracks the scale at which the protection is currently weakest.
To meet (iii), each step encodes a single downscaled copy, so SRIM costs the same as EOT and does not need the editor's preprocessing.
We also provide theoretical analysis showing that, under a linearized spectral model, the best worst-case protection degrades only logarithmically with the width of the scale range, whereas averaging over scales does not attain it (\cref{sec:method-theory}).

In summary, the main contributions of this paper are as follows:
\begin{itemize}
\item We revisit image immunization under the resize that editors apply to shared photos, and reveal the root cause of its failure: the protection is optimized for a scale, or an average of scales, that the editor does not use.
\item Based on our analysis, we propose SRIM, which covers the whole range of scales with anchors and spends each step on the currently weakest one, as cheaply as EOT.
\item We theoretically show that the best worst-case protection under an unknown resize degrades only logarithmically with the scale range, and that averaging over scales does not attain it.
\item Experiments on full-resolution photos and four editors, including FLUX.2-dev, show that SRIM raises the worst-case disruption of edits to 2.3 to 2.4 times that of the strongest published protection on the FLUX.2 editors.
\end{itemize}
\section{Related Work}
\label{sec:related}

\paragraph{Image immunization against malicious editing}
PhotoGuard~\cite{salman2023photoguard} introduced protective perturbations against latent diffusion editing, either by pushing the encoder latent toward a target or by attacking the full diffusion process.
AdvDM~\cite{liang2023advdm} and Mist~\cite{liang2023mist} maximize the denoising loss and add a textural latent loss, and score distillation makes these attacks cheaper~\cite{xue2024sds}.
Guo \emph{et al.}~\cite{guo2025pca} showed that attacking only the posterior of the variational autoencoder (VAE) already disrupts latent diffusion~\cite{rombach2022ldm} editing and transfers across models that share the autoencoder.
DiffusionGuard~\cite{choi2025diffusionguard} targets inpainting with mask augmentation, FaceLock~\cite{wang2025facelock} optimizes against face recognition features, and DeContext~\cite{shen2025decontext} weakens the attention pathways that carry context in diffusion transformer editors.
Related protections target personalization~\cite{vanle2023antidreambooth,shan2023glaze,yang2025antipure,guo2026artwork,guo2024laaca,ahn2026imperceptible} and cross-model transfer~\cite{zhang2025transferable,wang2024uap}; recent surveys cover the area~\cite{zhang2026survey}.
All of these methods compute the perturbation at the resolution at which the editor will see the image.
Our analysis concerns the step they leave out: the resize between the shared photo and the editor's input.

\paragraph{Robustness of protections}
Protections are known to be fragile under compression, purification, and resizing~\cite{nie2022diffpure,sun2026doprotective,purifyonce2026,mustafa2020superres}.
The standard remedy is expectation over transformation (EOT)~\cite{athalye2018eot}, and EOLT learns the transformation distribution instead of sampling it uniformly~\cite{yao2025eolt}.
AntiPure adds frequency guidance to resist diffusion purification~\cite{yang2025antipure}.
We show that for an unknown resize the relevant quantity is the worst case over scales rather than the average, derive its optimal value, and give a method that targets it as cheaply as EOT.

\paragraph{Resampling and adversarial perturbations}
Image-scaling attacks exploit downsampling in the opposite direction: they hide a target image that appears only after the victim resizes the input~\cite{xiao2019scaling,quiring2020scaling}.
Low-frequency adversarial perturbations are known to transfer better and to survive defenses that remove high frequencies~\cite{guo2019lowfreq,zhu2025lowfreq,weng2026frequencies}.
Our transfer-function analysis explains why low-frequency content matters for protection against editors in particular: the editor discards the top of its own frequency band, so energy placed there is wasted even before any resize.
\section{Problem Setup}
\label{sec:setup}

\paragraph{Threat model}
An owner shares a photo $\bm{x}\in[0,1]^{3\times H\times W}$ at its native resolution, typically 10 to 30 megapixels.
Before sharing, the owner adds a protective perturbation $\bm{\delta}$ with $\|\bm{\delta}\|_\infty\le\epsilon$ and rounds the result to 8 bits.
An adversary later edits the published image with an instruction-guided editor.
The editor resizes its input to its working resolution, encodes it with a VAE encoder $\mathcal{E}$, and generates the edit conditioned on the latent and the instruction.
The owner knows neither which editor will be used nor its preprocessing, so the downscaling factor $s$ is unknown.
Following prior work on latent editors~\cite{guo2025pca,salman2023photoguard}, we assume gray-box access to the VAE encoder of the editor family (the encoders of open editors are public), and evaluate how the protection affects the full editing pipeline.

\paragraph{Resize model}
We write the editor's preprocessing as $\mathcal{R}_s$: an antialiasing low-pass filter followed by decimation that reduces the pixel count by $s^2$.
Common choices (bicubic, bilinear, area, and Lanczos) differ in their kernels; we optimize with randomly drawn differentiable kernels and evaluate with the Lanczos filter of the Python Imaging Library, which none of the methods sees during optimization.
The owner covers the scale range $[s_{\min},s_{\max}]$ implied by the photo size and the range of editor working resolutions, $s_{\min}=\sqrt{HW/P_{\max}}$ and $s_{\max}=\sqrt{HW/P_{\min}}$ with $P_{\min}=0.25$ and $P_{\max}=4.2$ megapixels.
For a 12-megapixel photo this gives $s\in[1.7,7.1]$.

\paragraph{Objective}
Let $\mathcal{E}_s(\bm{x})=\mathcal{E}(\mathcal{R}_s{\bm{x}})$ be the latent the editor computes at scale $s$.
We measure the effect of $\bm{\delta}$ at scale $s$ by the relative latent deviation
\begin{equation}
  D_s(\bm{\delta})=\frac{\|\mathcal{E}_s(\bm{x}+\bm{\delta})-\mathcal{E}_s(\bm{x})\|_2}{\|\mathcal{E}_s(\bm{x})\|_2},
  \label{eq:dev}
\end{equation}
and the success of the protection by how much the edit of the protected photo departs from the edit of the clean photo under the same instruction and random seed (\cref{sec:exp}).
Since an edit is disrupted only once the deviation crosses a threshold, the owner cares about the \emph{weakest} scale:
\begin{equation}
  \max_{\|\bm{\delta}\|_\infty\le\epsilon}\ \min_{s\in[s_{\min},s_{\max}]} D_s(\bm{\delta}).
  \label{eq:minimax}
\end{equation}
\section{Revisiting Existing Protections}
\label{sec:theory}

Before presenting our method, we first revisit existing protections to understand why they fall short under an unknown resize.
We analyze them in a linearized spectral model.

\paragraph{Spectral model}
We write frequencies in units of the Nyquist frequency, so that the native grid carries $|\omega|\le1$.
An ideal antialiased downscaling by $s$ keeps the components with $|\omega|\le 1/s$.
It moves them to the working-grid frequency $\nu=s|\omega|$ and preserves their amplitude.
We make two assumptions, both checked empirically below.

\begin{assumption}[Local linearity]\label{as:lin}
For $\|\bm{\delta}\|_\infty\le\epsilon$, $\mathcal{E}_s(\bm{x}+\bm{\delta})-\mathcal{E}_s(\bm{x})=J_s\mathcal{R}_s{\bm{\delta}}+o(\epsilon)$.
\end{assumption}

\begin{assumption}[Isotropic transfer]\label{as:transfer}
The sensitivity of the deviation to a working-grid component depends on its radial frequency only, through a transfer function $g:[\nu_0,1]\to{\mathbb{R}}_{\ge0}$ with peak value $g^\star=\max_\nu g(\nu)$ at $\nu^\star$, where $\nu_0$ is the lowest frequency on the working grid.
\end{assumption}

Let $\hat{\bm{\delta}}$ be the spectrum of $\bm{\delta}$, normalized so that $\int|\hat{\bm{\delta}}|^2=\overline{\delta^2}\le\epsilon^2$ (mean energy per pixel).
Under these assumptions, the deviation \eqref{eq:dev} at scale $s$ collects the surviving spectrum weighted by the transfer function:
\begin{equation}
  D_s(\bm{\delta})^2=\int_{|\omega|\le 1/s}|\hat{\bm{\delta}}(\omega)|^2\,g(s|\omega|)^2\,d\omega .
  \label{eq:spectral}
\end{equation}
We compare every perturbation with the per-scale optimum $\mathrm{OPT}_s=\max_{\overline{\delta^2}\le\epsilon^2}D_s=g^\star\epsilon$ and write $\rho_s(\bm{\delta})=D_s(\bm{\delta})^2/\mathrm{OPT}_s^2\in[0,1]$.
In the proofs, $p(\omega)=|\hat{\bm{\delta}}(\omega)|^2/\overline{\delta^2}$ denotes the normalized spectral density of the perturbation ($\int p=1$), so that \eqref{eq:spectral} gives $\rho_s(\bm{\delta})=\overline{\delta^2}\epsilon^{-2}\int_{|\omega|\le1/s}p(\omega)\,g(s|\omega|)^2/g^{\star2}\,d\omega\le\int_{|\omega|\le 1/s}p(\omega)\,g(s|\omega|)^2/g^{\star2}\,d\omega$, with equality at full energy $\overline{\delta^2}=\epsilon^2$; upper bounds use the inequality and constructions use full energy.

\noindent\textbf{Settings.}
To check \cref{as:transfer}, we add band-pass noise at the working resolution of FLUX.2-klein to eight development photos and record the latent deviation and the edit LPIPS for each frequency band.

\begin{figure}[!t]
\centering
\includegraphics[width=\columnwidth]{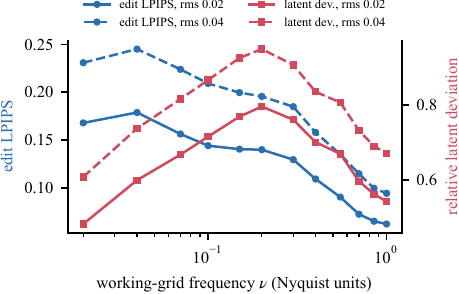}
\caption{Measured transfer of FLUX.2-klein to band-pass noise added at its working resolution (eight development photos at a 1024-pixel long side; two noise levels).
Edit LPIPS (left axis) is largest at the lowest frequencies and decays toward the Nyquist frequency, whereas the VAE latent deviation (right axis) peaks near $\nu\approx0.2$.}
\label{fig:transfer}
\end{figure}

\noindent\textbf{Results.}
As shown in \cref{fig:transfer}, the editor responds most to low working-grid frequencies and little to the top of its band.
The transfer of the edit is therefore low-pass, while the latent deviation peaks at $\nu\approx0.2$.

\subsection{The Limitations of Native-Resolution Protection}
\label{sec:rev-native}
Computing the protection at the native resolution relies on the latent assumption that a perturbation which moves the latent of the full-resolution photo also moves the latent the editor computes.
However, sign-gradient steps~\cite{madry2018pgd,jia2022fastat,lau2023ijsat,liu2022mutualat} make such a perturbation close to white noise, and the resize discards most of it.

\begin{theorem}[Decay]\label{thm:decay}
If $\hat{\bm{\delta}}$ is flat on the native band, then $\rho_s(\bm{\delta})=\bar g^2/(g^{\star 2}s^2)$ in two dimensions, where $\bar g^2$ is the mean of $g^2$ over the working band.
\end{theorem}
\begin{proof}
A flat spectrum on the native band $B=\{|\omega|\le1\}$ is $p=1/|B|$.
Then $\rho_s=\frac{1}{|B|}\int_{|\omega|\le1/s}g(s|\omega|)^2/g^{\star2}\,d\omega$.
Substituting $\nu=s\omega$ in two dimensions gives $d\omega=d\nu/s^2$ and maps the disk of radius $1/s$ onto $B$, so $\rho_s=\frac{1}{s^2}\frac{1}{|B|}\int_B g(|\nu|)^2/g^{\star2}\,d\nu=\bar g^2/(g^{\star2}s^2)$.
\end{proof}

Two effects compound.
The resize keeps only a fraction $1/s^2$ of the energy.
What remains is spread over the whole working band, although the editor responds mostly to part of it, so $\rho_1=\bar g^2/g^{\star2}<1$ already without resizing.
In line with \cref{thm:decay}, on development photos of 2 megapixels, a native-resolution perturbation that moves the latent strongly changes the edit by only 0.157 LPIPS at $s=1$, and by 0.077 and 0.044 at $s=2$ and $3$.

\subsection{The Limitations of Fixed-Resolution Protection}
\label{sec:rev-fixed}
Computing the protection on a downscaled copy relies on the latent assumption that the editor works at the resolution of that copy.

\begin{theorem}[Lifting]\label{thm:lift}
For box or other nonnegative kernels, $\|\mathcal{R}_s{\bm{\delta}}\|_\infty\le\|\bm{\delta}\|_\infty$.
Hence no native perturbation exceeds the best perturbation computed at the working resolution, and upsampling that perturbation by nearest-neighbor interpolation attains it.
\end{theorem}
\begin{proof}
For a kernel $k_s\ge0$ with $\sum k_s=1$, each output pixel of $\mathcal{R}_s{\bm{\delta}}$ is a convex combination of input pixels, so $\|\mathcal{R}_s{\bm{\delta}}\|_\infty\le\|\bm{\delta}\|_\infty\le\epsilon$.
Under \cref{as:lin}, the deviation caused by $\bm{\delta}$ at scale $s$ is the deviation caused by the working-resolution perturbation $\bm{u}=\mathcal{R}_s{\bm{\delta}}$, which lies in the $\ell_\infty$ ball of radius $\epsilon$; hence it cannot exceed $\max_{\|\bm{u}\|_\infty\le\epsilon}\|J_s\bm{u}\|$.
For integer $s$ and the box kernel, nearest-neighbor upsampling $\mathcal{U}_s$ satisfies $\mathcal{R}_s\mathcal{U}_s{\bm{u}}=\bm{u}$ and $\|\mathcal{U}_s{\bm{u}}\|_\infty=\|\bm{u}\|_\infty$, so $\bm{\delta}=\mathcal{U}_s{\bm{u}}^\star$ attains the maximum.
\end{proof}

\begin{theorem}[Window]\label{thm:window}
A perturbation computed at scale $s_0$ and upsampled has $\rho_s=g(\nu^\star s/s_0)^2/g^{\star2}$ for $s\le s_0/\nu^\star$ and $\rho_s=0$ beyond.
\end{theorem}
\begin{proof}
Computing the protection at scale $s_0$ and attaining $\mathrm{OPT}_{s_0}$ places the energy at working frequency $\nu^\star$ of the $s_0$ grid, i.e.\ at native frequency $\nu^\star/s_0$.
At scale $s$ this component survives iff $\nu^\star/s_0\le1/s$, and then sits at working frequency $s\nu^\star/s_0$, giving $\rho_s=g(\nu^\star s/s_0)^2/g^{\star2}$.
\end{proof}

Protecting at a fixed working resolution is therefore optimal at one scale and decays on both sides of it.
For a low-pass $g$, \cref{thm:window} predicts that a protection computed at one megapixel is strongest at the smallest downscaling, where its content sits at low working-grid frequencies, and decays as $s$ grows.
\Cref{tab:perscale-klein} confirms this for the published protections applied at one megapixel: for PCA@1MP, the edit LPIPS on FLUX.2-klein falls from 0.189 at $s=2$ to 0.077 at $s=8$.

\begin{figure}[!t]
\centering
\includegraphics[width=\columnwidth]{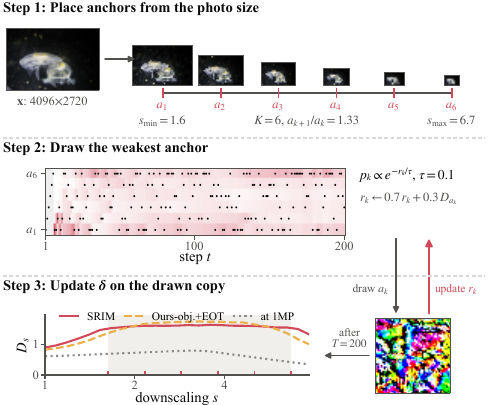}
\caption{SRIM on one test photo (FLUX.2 VAE, $\epsilon=4/255$); in general, it consists of three steps.
Step 1: Place anchors from the photo size. The photo size sets the range $[s_{\min},s_{\max}]$, which $K=6$ geometric anchors cover; copies are drawn to scale.
Step 2: Draw the weakest anchor. Shading gives $p_k$ at each step (gray: the first $K$ steps visit every anchor once), and dots mark the drawn anchor; the two ends are drawn most often.
Step 3: Update $\bm{\delta}$ on the drawn copy. A sign-gradient step on that copy updates $\bm{\delta}$ (center crop, amplified) and $r_k$.
After $T=200$ steps, the worst latent deviation $D_s$ over the range (shaded) is 1.54, against 1.37 for EOT with our objective and 0.44 for our objective at one megapixel.}
\label{fig:pipeline}
\end{figure}

\subsection{The Limitations of Averaging over Scales}
\label{sec:rev-eot}
The standard remedy, EOT~\cite{athalye2018eot}, draws the scale at random and maximizes the average protection.
It relies on the latent assumption protecting the average scale protects every scale.

\begin{figure*}[!t]
\centering
\includegraphics[width=\textwidth]{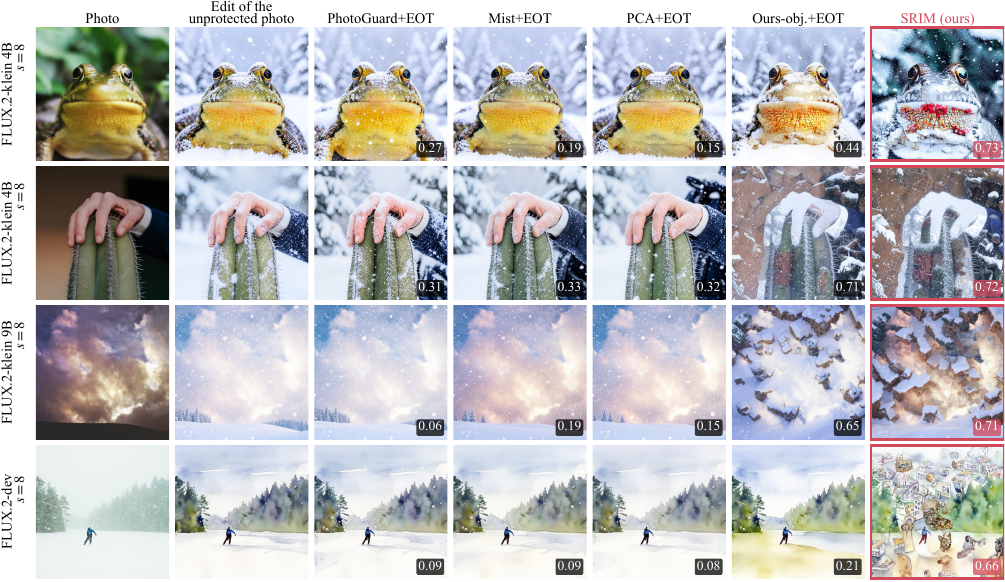}
\caption{Selected examples of real edits by the FLUX.2 editors at the largest downscaling ($s=8$).
The instruction is ``Make it snowy winter'' in rows 1 to 3 and ``Turn this photo into a watercolor painting'' in row 4; cells are center-cropped to squares.
Badges give the edit LPIPS to the edit of the unprotected photo.}
\label{fig:qual}
\end{figure*}

\begin{theorem}[EOT is not minimax]\label{thm:eot}
The maximizer of the average of $\rho_s$ over log-uniform $s$ concentrates its energy at a single native frequency $\omega^\dagger$.
Its worst case is $\min_s g(s\omega^\dagger)^2/g^{\star2}$, which tends to $0$ when $\ln S$ exceeds the log-width of the support of $g$, where $S=s_{\max}/s_{\min}$.
\end{theorem}
\begin{proof}
The log-average of $\rho_s$ equals $\int p(\omega)a(\omega)\,d\omega$ with $a(\omega)=\frac{1}{\ln S}\int g(s|\omega|)^2/g^{\star2}\,\mathbf 1[s|\omega|\le1]\,ds/s$, which is linear in $p$; over the simplex it is maximized by a point mass at $\omega^\dagger\in\arg\max a$.
For this perturbation $\rho_s=g(s\omega^\dagger)^2/g^{\star2}$.
If $g$ vanishes outside $[\nu_c,\nu_d]$ and $\ln S>\ln(\nu_d/\nu_c)$, then $s\omega^\dagger$ leaves $[\nu_c,\nu_d]$ for some $s$ in the range, and the worst case is $0$.
When $g$ is positive but small outside a band, the same argument gives a worst case bounded by the value of $g^2/g^{\star2}$ outside the band.
EOT in practice averages $D_s$ rather than $\rho_s=D_s^2/\mathrm{OPT}_s^2$.
Since $D_s$ is concave in $p$, the maximizer is no longer a point mass.
By Jensen's inequality, however, it still weights scales by their average gain and does not equalize the worst case.
\end{proof}

EOT therefore sacrifices the scales at the ends of the range.
With our objective, EOT is flat in the middle of the range on FLUX.2-klein (0.473 to 0.518 for $s\le6$) but drops to 0.382 at $s=8$, the largest downscaling.

The above analyses reveal that the root cause is the same in all three cases: each protection is optimized for a scale, or for an average of scales, that does not represent the scale the editor uses.
This motivates us to optimize the protection for the scale at which it is currently weakest, rather than for one expected scale or for the average.
\section{SRIM: Worst-Case Scale-Robust Protection}
\label{sec:method}

\begin{algorithm}[!t]
\caption{SRIM}\label{alg:method}
\begin{algorithmic}[1]
\REQUIRE photo $\bm{x}$, encoder $\mathcal{E}$, budget $\epsilon$, steps $T$, temperature $\tau$
\STATE anchors $a_{1..K}$ from the photo size; $\bm{\delta}\sim\mathcal U[-\epsilon,\epsilon]$; $r_k\gets$ unset
\FOR{$t=1$ \TO $T$}
  \STATE $k\gets$ an unseen anchor if any, else draw $k$ with $p_k\propto e^{-r_k/\tau}$
  \STATE $s\gets a_k e^{u}$, $u\sim\mathcal U[-0.1,0.1]$; draw kernel $\kappa$
  \STATE $d\gets D_s(\bm{\delta})$ with resize $\mathcal R_{s,\kappa}$ \hfill\eqref{eq:dev}
  \STATE $r_k\gets d$ if unset, else $0.7\,r_k+0.3\,d$
  \STATE $\bm{\delta}\gets\Pi_\epsilon\bigl(\bm{\delta}+\alpha_t\,\mathrm{sign}\nabla_{\bm{\delta}}d\bigr)$; clip $\bm{x}+\bm{\delta}$ to $[0,1]$
\ENDFOR
\STATE \textbf{return} 8-bit rounding of $\bm{x}+\bm{\delta}$
\end{algorithmic}
\end{algorithm}

The analysis of \cref{sec:theory} suggests a different starting point: rather than protecting the scale that the editor is expected to use, or the average of all scales, a protection should be optimized for the scale at which it is currently weakest.
This step is not obvious because existing protections treat the resize like any other random augmentation, and a direct minimax over scales would evaluate every scale at every step.
In \cref{fig:pipeline}, SRIM makes the worst case affordable with three components, which meet requirements (i) to (iii) of \cref{sec:intro} in turn.

\begin{table*}[!t]
\centering
\caption{Protection of full-resolution photos (median 11.3 megapixels) against downscaling by an unknown factor ($\epsilon=4/255$).
Edit LPIPS between the edits of the protected and the clean photo (higher is better), averaged over photos and instructions.
Avg.\ and Worst average or minimize it over the downscaling factors of each editor (\cref{sec:exp}).
``Ours-obj.'' is our objective without the worst-case scale sampling; on FLUX.2-dev, PCA@1MP represents the protections applied at one megapixel (--: not evaluated); native-resolution application of the baselines exceeds a 48\,GB GPU.}
\label{tab:main}
\setlength{\tabcolsep}{5.8mm}
\begin{tabular}{lcccccccc}
\toprule
 & \multicolumn{2}{c}{FLUX.2-klein 4B} & \multicolumn{2}{c}{FLUX.2-klein 9B} & \multicolumn{2}{c}{FLUX.2-dev} & \multicolumn{2}{c}{InstructPix2Pix} \\
\cmidrule(lr){2-3} \cmidrule(lr){4-5} \cmidrule(lr){6-7} \cmidrule(lr){8-9}
Method & Avg. & Worst & Avg. & Worst & Avg. & Worst & Avg. & Worst \\
\midrule
PhotoGuard@1MP & 0.131 & 0.077 & 0.133 & 0.082 & -- & -- & 0.173 & 0.107 \\
PhotoGuard+EOT & 0.240 & 0.171 & 0.224 & 0.161 & 0.221 & 0.157 & 0.218 & 0.163 \\
Mist@1MP & 0.163 & 0.097 & 0.157 & 0.096 & -- & -- & 0.191 & 0.129 \\
Mist+EOT & 0.283 & 0.192 & 0.269 & 0.184 & 0.238 & 0.166 & 0.238 & 0.183 \\
PCA@1MP & 0.130 & 0.077 & 0.130 & 0.080 & 0.108 & 0.072 & 0.183 & 0.115 \\
PCA+EOT & 0.241 & 0.163 & 0.227 & 0.153 & 0.243 & 0.161 & 0.245 & 0.174 \\
Ours-obj.@1MP & 0.144 & 0.075 & 0.138 & 0.078 & -- & -- & 0.220 & 0.152 \\
Ours-obj.+EOT & 0.474 & 0.382 & 0.463 & 0.367 & 0.426 & 0.319 & 0.331 & 0.250 \\
\midrule
{SRIM} & \textbf{0.489} & \textbf{0.463} & \textbf{0.479} & \textbf{0.450} & \textbf{0.442} & \textbf{0.386} & \textbf{0.332} & \textbf{0.269} \\
\bottomrule
\end{tabular}
\end{table*}

\paragraph{Anchors cover the range (i)}
We place anchors $a_1<\dots<a_K$ on a geometric grid with ratio $1.4$ over the range $[s_{\min},s_{\max}]$ that the photo size implies (\cref{sec:setup}); a 12-megapixel photo has $K=6$.
At each step, we perturb the drawn anchor by a random factor in $[e^{-0.1},e^{0.1}]$ and draw the resampling kernel uniformly from bicubic, bilinear, and area filters, so the protection overfits neither to a grid of scales nor to one kernel.
A natural alternative is to draw the scale log-uniformly, as EOT does; however, the ends of the range then receive little weight and are the first to fail (\cref{thm:eot}).
With anchors at both ends, even uniform sampling over anchors gives the extreme scales the same weight as the middle ones.

\paragraph{Weakness tracking protects the worst scale (ii)}
For each anchor we keep an exponential moving average $r_k$ of the deviation $D_{a_k}$ observed when it was last drawn.
Each anchor is drawn once at the start; afterwards anchor $k$ is drawn with probability $\propto\exp(-r_k/\tau)$.
A natural alternative is a direct minimax that evaluates all $K$ anchors at every step; it reaches the same objective but multiplies the cost by $K$ and by the pixel count of the larger anchors.
Low temperature $\tau$ approaches the minimax objective, while high temperature samples the anchors uniformly. and \cref{sec:exp-analysis} shows that the result is insensitive to $\tau$.
\Cref{alg:method} summarizes the procedure.

\paragraph{One copy per step keeps the cost (iii)}
We maximize the relative latent deviation \eqref{eq:dev} of the editor family's VAE encoder, an untargeted objective that needs neither a target image nor the denoiser.
Because each step encodes one downscaled copy, a step costs less than one encoder pass at the native resolution, and the memory stays below that of encoding a 5.1-megapixel image (the largest anchor with its jitter).
All settings were fixed on a development set before testing: $T=200$ steps, step size $\alpha_t=\epsilon/4$ with cosine decay, $\tau=0.1$, and $\epsilon=4/255$ unless stated otherwise.

\subsection{Theoretical Analysis of SRIM}
\label{sec:method-theory}
We now ask what the worst case over the scale range can attain at best, using the model of \cref{sec:theory}.
Let $S=s_{\max}/s_{\min}$ and define the log-sensitivity mass $C_g=\int_{\nu_0}^1 g(\nu)^2/g^{\star2}\,d\nu/\nu$, the area under $g^2$ on a logarithmic frequency axis.
The following result justifies covering the range with a geometric grid of anchors.

\begin{theorem}[Minimax value]\label{thm:minimax}
For every perturbation, $\min_{s\in[s_{\min},s_{\max}]}\rho_s\le C_g/\ln S$.
Suppose that $g^2\ge g^{\star2}/2$ on a log-frequency interval of width $w$.
Spreading the energy log-uniformly over the native frequencies that this interval occupies across all scales gives $\rho_s\ge w/\bigl(2(\ln S+w)\bigr)$ for every $s$ in the range.
\end{theorem}
\begin{proof}
\emph{Upper bound.}
Average $\rho_s$ over $\ln s$ uniform on $[\ln s_{\min},\ln s_{\max}]$:
\begin{align*}
\frac{1}{\ln S}\int_{s_{\min}}^{s_{\max}}\rho_s\frac{ds}{s}
&\le\int p(\omega)\,A(\omega)\,d\omega\\
&\le\int p(\omega)\,\frac{C_g}{\ln S}\,d\omega=\frac{C_g}{\ln S},
\end{align*}
with $A(\omega)=\frac{1}{\ln S}\int_{s_{\min}}^{s_{\max}}\frac{g(s|\omega|)^2}{g^{\star2}}\mathbf 1[s|\omega|\le1]\frac{ds}{s}$;
the bound $A(\omega)\le C_g/\ln S$ follows by substituting $\nu=s|\omega|$ ($ds/s=d\nu/\nu$) and extending the range to $[\nu_0,1]$.
The minimum over $s$ is at most the average.

\emph{Construction.}
Let $[\nu_a,\nu_b]$ with $\ln(\nu_b/\nu_a)=w$ be an interval on which $g^2\ge g^{\star2}/2$.
Spread the energy log-uniformly over radial native frequencies $\Omega=[\nu_a/s_{\max},\nu_b/s_{\min}]$, whose log-width is $L=w+\ln S$.
For any $s\in[s_{\min},s_{\max}]$, the native frequencies that land in $[\nu_a,\nu_b]$ form the interval $[\nu_a/s,\nu_b/s]\subseteq\Omega$, of log-width $w$, so they carry a fraction $w/L$ of the energy, each with $g^2\ge g^{\star2}/2$.
Hence $\rho_s\ge w/(2L)$.
\end{proof}

\Cref{thm:minimax} indicates that the best worst-case protection degrades only logarithmically with the range of scales, and that some loss relative to the per-scale optimum is unavoidable once $\ln S$ exceeds the width of the sensitive band.
The construction that attains the bound gives every part of the log-scale range the same weight, which is what the geometric anchor grid of SRIM does, while weakness tracking moves effort toward the scales that fall behind.
By \cref{thm:eot}, averaging over scales does not reach this value; this is the gap that SRIM closes in practice (\cref{sec:exp-main}).

\paragraph{Relation to prior work}
Similar techniques have been explored in EOT~\cite{athalye2018eot} and EOLT~\cite{yao2025eolt}, which sample or learn a distribution over transformations; the key difference is that both maximize the average over that distribution, whereas SRIM targets the weakest scale.
SRIM borrows the exponentiated group weights of group distributionally robust optimization~\cite{sagawa2020dro}; here the groups are anchor scales derived from the photo size, and each weight is updated from the single copy drawn at that step.
\section{Experiments}
\label{sec:exp}

\begin{figure*}[p]
\centering
\includegraphics[width=\textwidth]{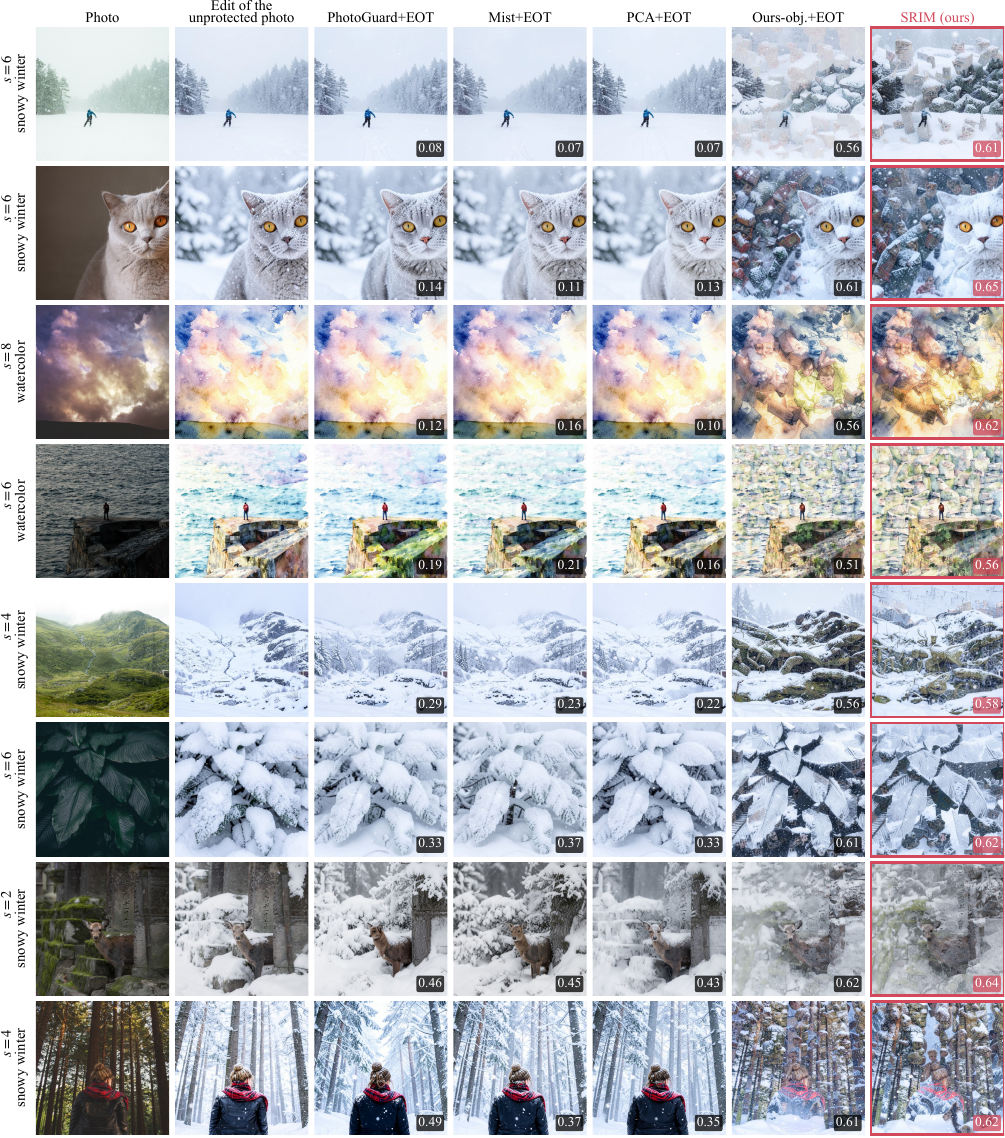}
\caption{Selected examples of real edits by FLUX.2-klein 4B (row labels give the downscaling factor and the instruction, ``Turn this photo into a watercolor painting'' or ``Make it snowy winter''; cells are center-cropped to squares).
Badges give the edit LPIPS to the edit of the unprotected photo.}
\label{fig:qual-flux2klein}
\end{figure*}

\begin{figure*}[p]
\centering
\includegraphics[width=\textwidth]{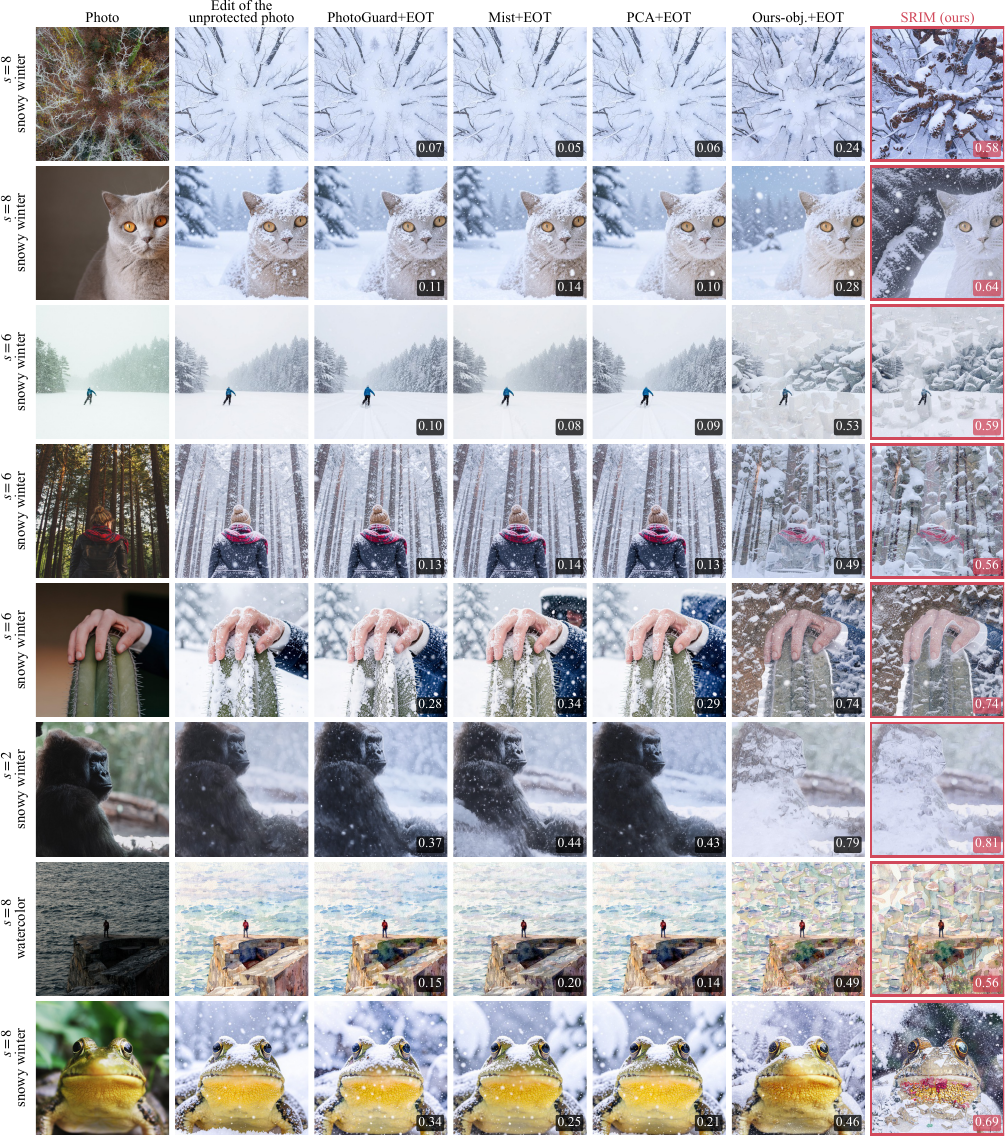}
\caption{Selected examples of real edits by FLUX.2-klein 9B (row labels give the downscaling factor and the instruction, ``Turn this photo into a watercolor painting'' or ``Make it snowy winter''; cells are center-cropped to squares).
Badges give the edit LPIPS to the edit of the unprotected photo.}
\label{fig:qual-flux2klein9b}
\end{figure*}

\begin{figure*}[p]
\centering
\includegraphics[width=\textwidth]{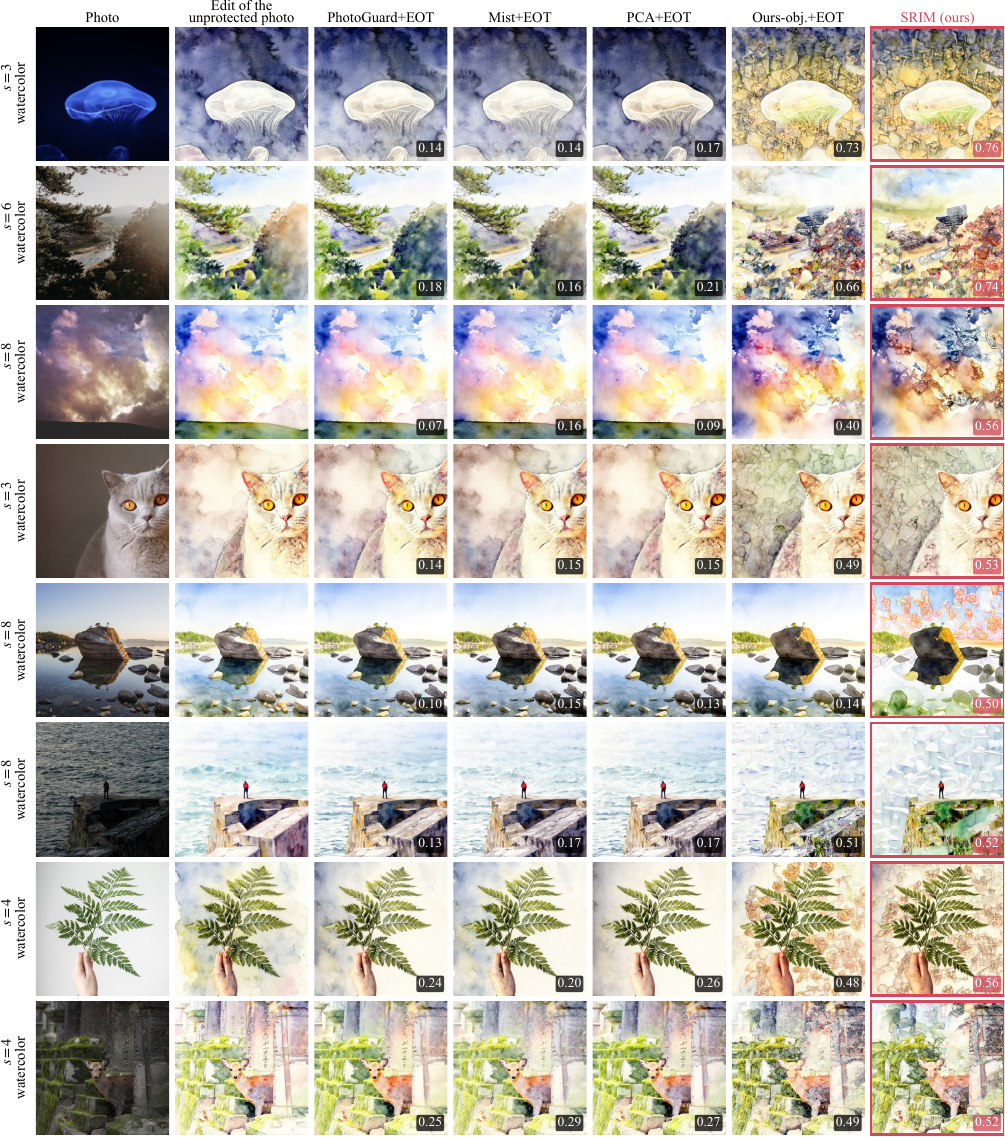}
\caption{Selected examples of real edits by FLUX.2-dev (row labels give the downscaling factor and the instruction, ``Turn this photo into a watercolor painting'' or ``Make it snowy winter''; cells are center-cropped to squares).
Badges give the edit LPIPS to the edit of the unprotected photo.}
\label{fig:qual-flux2dev}
\end{figure*}

\begin{table}[!tb]
\centering
\caption{FLUX.2-klein 4B, per downscaling factor (160 photos, two instructions). LPIPS, higher is better; PSNR, SSIM, and CLIP-I between the two edits, lower is better.}
\label{tab:perscale-klein}
\resizebox{\linewidth}{!}{\begin{tabular}{lcccccccc}
\toprule
 & \multicolumn{5}{c}{LPIPS $\uparrow$ at $s=$} & \multicolumn{3}{c}{mean over $s$} \\
\cmidrule(lr){2-6} \cmidrule(lr){7-9}
Method & 2 & 3 & 4 & 6 & 8 & PSNR $\downarrow$ & SSIM $\downarrow$ & CLIP-I $\downarrow$ \\
\midrule
PhotoGuard@1MP & 0.186 & 0.164 & 0.131 & 0.097 & 0.077 & 23.98 & 0.863 & 0.981 \\
PhotoGuard+EOT & 0.281 & 0.267 & 0.260 & 0.224 & 0.171 & 20.11 & 0.770 & 0.951 \\
Mist@1MP & 0.222 & 0.201 & 0.168 & 0.128 & 0.097 & 22.65 & 0.835 & 0.977 \\
Mist+EOT & 0.323 & 0.319 & 0.307 & 0.274 & 0.192 & 19.46 & 0.744 & 0.949 \\
PCA@1MP & 0.189 & 0.163 & 0.129 & 0.094 & 0.077 & 23.97 & 0.864 & 0.981 \\
PCA+EOT & 0.289 & 0.276 & 0.258 & 0.220 & 0.163 & 19.77 & 0.769 & 0.954 \\
Ours-obj.@1MP & 0.211 & 0.185 & 0.148 & 0.100 & 0.075 & 23.42 & 0.850 & 0.980 \\
Ours-obj.+EOT & \textbf{0.473} & \textbf{0.494} & \textbf{0.518} & 0.502 & 0.382 & 15.96 & 0.581 & 0.871 \\
\midrule
{SRIM} & 0.463 & 0.487 & 0.514 & \textbf{0.515} & \textbf{0.466} & \textbf{15.60} & \textbf{0.566} & \textbf{0.858} \\
\bottomrule
\end{tabular}}
\end{table}

\begin{table}[!tb]
\centering
\caption{FLUX.2-klein 9B, per downscaling factor (160 photos, two instructions). LPIPS, higher is better; PSNR, SSIM, and CLIP-I between the two edits, lower is better.}
\label{tab:perscale-klein9b}
\resizebox{\linewidth}{!}{\begin{tabular}{lcccccccc}
\toprule
 & \multicolumn{5}{c}{LPIPS $\uparrow$ at $s=$} & \multicolumn{3}{c}{mean over $s$} \\
\cmidrule(lr){2-6} \cmidrule(lr){7-9}
Method & 2 & 3 & 4 & 6 & 8 & PSNR $\downarrow$ & SSIM $\downarrow$ & CLIP-I $\downarrow$ \\
\midrule
PhotoGuard@1MP & 0.180 & 0.167 & 0.136 & 0.100 & 0.082 & 24.52 & 0.828 & 0.982 \\
PhotoGuard+EOT & 0.259 & 0.249 & 0.242 & 0.208 & 0.161 & 21.10 & 0.748 & 0.962 \\
Mist@1MP & 0.206 & 0.195 & 0.168 & 0.123 & 0.096 & 23.53 & 0.805 & 0.980 \\
Mist+EOT & 0.308 & 0.304 & 0.293 & 0.255 & 0.184 & 19.97 & 0.716 & 0.958 \\
PCA@1MP & 0.177 & 0.164 & 0.134 & 0.095 & 0.080 & 24.67 & 0.831 & 0.983 \\
PCA+EOT & 0.267 & 0.260 & 0.247 & 0.205 & 0.153 & 20.71 & 0.738 & 0.965 \\
Ours-obj.@1MP & 0.193 & 0.177 & 0.142 & 0.102 & 0.078 & 24.40 & 0.817 & 0.982 \\
Ours-obj.+EOT & \textbf{0.458} & \textbf{0.486} & \textbf{0.511} & 0.493 & 0.367 & 16.28 & 0.537 & 0.881 \\
\midrule
{SRIM} & 0.451 & 0.481 & 0.504 & \textbf{0.507} & \textbf{0.450} & \textbf{15.89} & \textbf{0.522} & \textbf{0.868} \\
\bottomrule
\end{tabular}}
\end{table}

\paragraph{Photos}
We use the Unsplash Lite dataset and keep photos with at least 12 megapixels and an aspect ratio of at most 1.8, split into people and scenes by keywords in their descriptions and ordered by the SHA-1 hash of their identifier.
We take 216 photos, half of each category, resize them to a width of 4096 pixels (median 11.3 megapixels), and center-crop them to multiples of 16 pixels.
The first 16 form the development set on which all settings were chosen, the next 40 were used once to confirm the method, and the remaining 160 form the test set.
We release the photo identifiers and the split.

\paragraph{Editors}
We attack four instruction-guided editors: FLUX.2-klein in its 4B and 9B versions (4 sampling steps, guidance 1.0), FLUX.2-dev (28 steps, guidance 4.0), and InstructPix2Pix~\cite{brooks2023ip2p} (50 steps, image guidance 1.5, text guidance 7.5).
The three FLUX.2 editors share one VAE, the surrogate for all of them, so the same protected photos are evaluated on each; InstructPix2Pix uses its own VAE.
The protected photo is downscaled with the Lanczos filter by $s\in\{2,3,4,6,8\}$ (working widths 2048 to 512 pixels, rounded to multiples of 16) and edited with fixed seeds and two fixed instructions, ``Turn this photo into a watercolor painting'' and ``Make it snowy winter''.
InstructPix2Pix, trained at 512 pixels, works at $s\in\{4,6,8\}$, and FLUX.2-dev at $s\in\{3,4,6,8\}$, where we compare with the EOT variants and PCA@1MP.
With ``Make it snowy winter'', FLUX.2-dev sometimes regenerates the scene instead of editing the photo; we keep these cases for all methods.

\paragraph{Metrics}
Protection is measured by how far the edit of the protected photo departs from the edit of the clean photo under the same instruction and seed.
We report LPIPS~\cite{zhang2018lpips} with the AlexNet backbone (higher is better), PSNR and SSIM~\cite{wang2004ssim} (lower is better), and CLIP image similarity~\cite{radford2021clip} with ViT-L/14 (lower is better), averaged over scales and in the worst case over scales, the quantity the owner cares about.

\paragraph{Baselines}
We compare with protections that, like ours, need only a public surrogate model and no training, each re-implemented on the same VAE as our method following the official code.
PhotoGuard~\cite{salman2023photoguard} minimizes the latent distance to the latent of a gray image, Mist~\cite{liang2023mist} to the latent of its target pattern resized to the input, and PCA~\cite{guo2025pca} ascends the posterior-collapse objective with $v=10^{-8}$.
We do not include protections that attack the denoiser~\cite{liang2023advdm,xue2024sds,shen2025decontext}, which would backpropagate through a transformer of 4 to 32 billion parameters at every step and are tied to one editor, whereas a VAE attack serves every editor that shares the autoencoder.
They can be applied at about one megapixel with the perturbation upsampled bicubically and clipped to the budget (\texttt{@1MP}), or at the native resolution, which exceeds the memory of a 48\,GB L40S GPU for 12-megapixel photos.
We also equip each baseline with EOT over the same scale range (\texttt{+EOT}), the standard remedy, and include EOT with our objective as the strongest such variant.
All methods share $\epsilon=4/255$, 200 steps, and the step-size schedule.

\begin{table}[!tb]
\centering
\caption{FLUX.2-dev, per downscaling factor (160 photos, two instructions). LPIPS, higher is better; PSNR, SSIM, and CLIP-I between the two edits, lower is better.}
\label{tab:perscale-dev}
\resizebox{\columnwidth}{!}{\begin{tabular}{lccccccc}
\toprule
 & \multicolumn{4}{c}{LPIPS $\uparrow$ at $s=$} & \multicolumn{3}{c}{mean over $s$} \\
\cmidrule(lr){2-5} \cmidrule(lr){6-8}
Method & 3 & 4 & 6 & 8 & PSNR $\downarrow$ & SSIM $\downarrow$ & CLIP-I $\downarrow$ \\
\midrule
PhotoGuard+EOT & 0.244 & 0.251 & 0.230 & 0.157 & 20.41 & 0.822 & 0.955 \\
Mist+EOT & 0.271 & 0.270 & 0.245 & 0.166 & 19.80 & 0.809 & 0.953 \\
PCA@1MP & 0.147 & 0.118 & 0.096 & 0.072 & 25.29 & 0.904 & 0.979 \\
PCA+EOT & 0.276 & 0.280 & 0.254 & 0.161 & 19.57 & 0.802 & 0.949 \\
Ours-obj.+EOT & \textbf{0.471} & \textbf{0.475} & 0.438 & 0.319 & 15.91 & 0.669 & 0.889 \\
\midrule
{SRIM} & 0.460 & 0.466 & \textbf{0.457} & \textbf{0.386} & \textbf{15.58} & \textbf{0.655} & \textbf{0.882} \\
\bottomrule
\end{tabular}}
\end{table}

\begin{table}[!tb]
\centering
\caption{InstructPix2Pix, per downscaling factor (160 photos, two instructions). LPIPS, higher is better; PSNR, SSIM, and CLIP-I between the two edits, lower is better.}
\label{tab:perscale-ip2p}
\resizebox{\columnwidth}{!}{\begin{tabular}{lcccccc}
\toprule
 & \multicolumn{3}{c}{LPIPS $\uparrow$ at $s=$} & \multicolumn{3}{c}{mean over $s$} \\
\cmidrule(lr){2-4} \cmidrule(lr){5-7}
Method & 4 & 6 & 8 & PSNR $\downarrow$ & SSIM $\downarrow$ & CLIP-I $\downarrow$ \\
\midrule
PhotoGuard@1MP & 0.249 & 0.164 & 0.107 & 24.27 & 0.790 & 0.960 \\
PhotoGuard+EOT & 0.266 & 0.226 & 0.163 & 22.01 & 0.744 & 0.948 \\
Mist@1MP & 0.263 & 0.180 & 0.129 & 23.85 & 0.771 & 0.952 \\
Mist+EOT & 0.291 & 0.242 & 0.183 & 22.15 & 0.724 & 0.942 \\
PCA@1MP & 0.262 & 0.172 & 0.115 & 24.17 & 0.783 & 0.960 \\
PCA+EOT & 0.306 & 0.256 & 0.174 & 21.70 & 0.732 & 0.948 \\
Ours-obj.@1MP & 0.296 & 0.212 & 0.152 & 23.46 & 0.742 & 0.952 \\
Ours-obj.+EOT & \textbf{0.403} & \textbf{0.341} & 0.250 & \textbf{21.12} & \textbf{0.638} & 0.934 \\
\midrule
{SRIM} & 0.390 & 0.337 & \textbf{0.269} & 21.18 & 0.640 & \textbf{0.932} \\
\bottomrule
\end{tabular}}
\end{table}

\begin{table*}[!t]
\centering
\caption{Budget sweep on FLUX.2-klein 4B (40 test photos, $s\in\{2,4,8\}$). Edit LPIPS (higher is better) and PSNR of the protected photo.}
\label{tab:sweep}
\setlength{\tabcolsep}{3mm}
\begin{tabular}{lcccccccccccc}
\toprule
 & \multicolumn{3}{c}{$\epsilon=2/255$} & \multicolumn{3}{c}{$\epsilon=4/255$} & \multicolumn{3}{c}{$\epsilon=6/255$} & \multicolumn{3}{c}{$\epsilon=8/255$} \\
\cmidrule(lr){2-4} \cmidrule(lr){5-7} \cmidrule(lr){8-10} \cmidrule(lr){11-13}
Method & Avg. & Worst & PSNR & Avg. & Worst & PSNR & Avg. & Worst & PSNR & Avg. & Worst & PSNR \\
\midrule
PhotoGuard+EOT & 0.146 & 0.096 & 44.3 & 0.195 & 0.144 & 39.0 & 0.226 & 0.175 & 36.1 & 0.247 & 0.198 & 33.8 \\
Mist+EOT & 0.175 & 0.103 & 43.8 & 0.242 & 0.166 & 38.5 & 0.288 & 0.211 & 35.6 & 0.321 & 0.246 & 33.3 \\
PCA+EOT & 0.155 & 0.093 & 44.0 & 0.191 & 0.130 & 38.6 & 0.216 & 0.154 & 35.6 & 0.234 & 0.173 & 33.3 \\
Ours-obj.+EOT & 0.334 & 0.266 & 43.4 & 0.408 & 0.360 & 37.9 & 0.443 & 0.408 & 35.0 & 0.468 & 0.441 & 32.8 \\
\midrule
{SRIM} & \textbf{0.356} & \textbf{0.336} & 43.3 & \textbf{0.429} & \textbf{0.398} & 37.8 & \textbf{0.461} & \textbf{0.430} & 34.9 & \textbf{0.482} & \textbf{0.455} & 32.7 \\
\bottomrule
\end{tabular}
\end{table*}

\subsection{Main results}
\label{sec:exp-main}

\Cref{tab:main} reports the main comparison, and \cref{tab:perscale-klein,tab:perscale-klein9b,tab:perscale-dev,tab:perscale-ip2p} break it down by downscaling factor on each editor.
Applied at about one megapixel, the published protections reach a worst case of at most 0.097 on the FLUX.2 editors, consistent with the scale window of \cref{thm:window}.
EOT recovers part of the loss, but the best published method with EOT, Mist+EOT, still reaches only 0.192 on FLUX.2-klein 4B, 0.184 on the 9B model, and 0.166 on FLUX.2-dev in the worst case.
SRIM raises the worst case to 0.463, 0.450, and 0.386 on these three editors, which is 2.3 to 2.4 times the best published protection.
It raises the worst case to 0.269 on InstructPix2Pix and gives the highest average on every editor.
The same protected photos transfer to the 9B model and to FLUX.2-dev, whose transformer is eight times larger, with a margin that shrinks with the size of the denoiser.
In \cref{fig:qual}, edits under published protections stay close to the unprotected edit, while SRIM disrupts them; \cref{fig:qual-flux2klein,fig:qual-flux2klein9b,fig:qual-flux2dev} give more examples on each FLUX.2 editor, eight photos each, at several downscaling factors and with both instructions.
As \cref{fig:perscale} shows, the ablation without worst-case sampling (Ours-obj.+EOT) matches or slightly exceeds SRIM in the middle of the range but drops at the extreme scales (worst case 0.382, 0.367, 0.319, and 0.250), as \cref{thm:eot} predicts.

\begin{table}[!tb]
\centering
\caption{Visibility of the protection at $\epsilon=4/255$ (40 test photos): PSNR and SSIM between the protected and the clean photo at the native resolution, and LPIPS between them after downscaling both to one megapixel.}
\label{tab:visibility}
\resizebox{\columnwidth}{!}{\begin{tabular}{lccc}
\toprule
Method & PSNR (native) $\uparrow$ & SSIM (native) $\uparrow$ & LPIPS at 1\,MP $\downarrow$ \\
\midrule
PhotoGuard@1MP & 39.26 & 0.966 & 0.087 \\
PhotoGuard+EOT & 39.20 & 0.956 & 0.067 \\
Mist@1MP & 39.11 & 0.965 & 0.100 \\
Mist+EOT & 38.64 & 0.955 & 0.153 \\
PCA@1MP & 38.98 & 0.963 & 0.117 \\
PCA+EOT & 38.72 & 0.954 & 0.124 \\
Ours-obj.@1MP & 39.32 & 0.965 & 0.111 \\
Ours-obj.+EOT & 38.05 & 0.953 & 0.202 \\
\midrule
{SRIM} & 37.98 & 0.954 & 0.207 \\
\bottomrule
\end{tabular}}
\end{table}

\begin{figure}[!tb]
\centering
\includegraphics[width=\columnwidth]{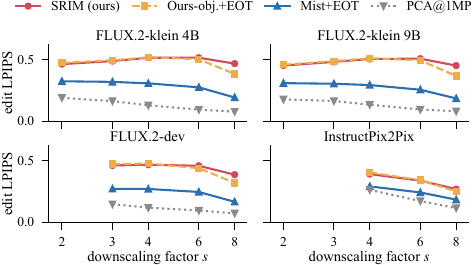}
\caption{Edit LPIPS at each downscaling factor on the four editors (mean over photos and instructions). PCA@1MP stands for the published protections applied at one megapixel, Mist+EOT for the strongest published protection with EOT.}
\label{fig:perscale}
\end{figure}

\begin{table}[!tb]
\centering
\caption{Worst-case edit LPIPS on FLUX.2-klein (40 test photos, $s\in\{2,4,8\}$) under different downscaling filters (top) and under purification of the downscaled input (bottom); the VAE round trip uses the FLUX.2 VAE, the surrogate of all protections.}
\label{tab:robust}
\resizebox{\columnwidth}{!}{\begin{tabular}{lccccc}
\toprule
 & \multicolumn{4}{c}{+EOT} & \\
\cmidrule(lr){3-6}
Processing & SRIM & Ours-obj. & Mist & PCA & PhotoGuard \\
\midrule
Lanczos (default) & \textbf{0.398} & 0.360 & 0.166 & 0.130 & 0.144 \\
Bicubic & \textbf{0.397} & 0.354 & 0.161 & 0.129 & 0.142 \\
Box & \textbf{0.397} & 0.348 & 0.163 & 0.126 & 0.140 \\
\midrule
JPEG 75 & \textbf{0.150} & 0.120 & 0.101 & 0.086 & 0.095 \\
\midrule
Gaussian blur, $r=1$ & \textbf{0.284} & 0.213 & 0.113 & 0.095 & 0.109 \\
Noise $4/255$ + JPEG 75 & \textbf{0.112} & 0.092 & 0.089 & 0.070 & 0.083 \\
Noise $8/255$ + JPEG 75 & \textbf{0.077} & 0.066 & 0.070 & 0.053 & 0.063 \\
VAE round trip & \textbf{0.400} & 0.340 & 0.143 & 0.120 & 0.133 \\
\bottomrule
\end{tabular}}
\end{table}

\begin{figure}[!tb]
\centering
\includegraphics[width=\columnwidth]{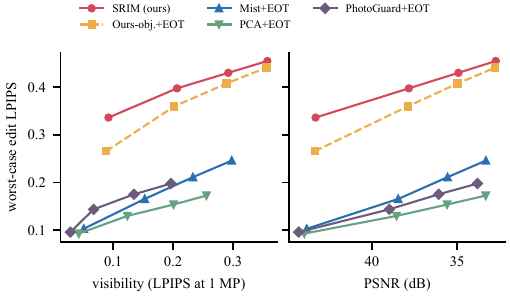}
\caption{Protection against visibility on FLUX.2-klein (40 test photos, $\epsilon\in\{2,4,6,8\}/255$): worst-case edit LPIPS over $s\in\{2,4,8\}$ against the LPIPS to the clean photo at one megapixel (left) and the PSNR at the native resolution (right).}
\label{fig:tradeoff}
\end{figure}

\subsection{Analysis}
\label{sec:exp-analysis}

\paragraph{Protection at equal budget and equal visibility}
At every $\ell_\infty$ budget we tested, SRIM gives the strongest protection, and it outperforms every published protection even with a quarter of their budget (0.336 at $\epsilon=2/255$ against at most 0.246 at $\epsilon=8/255$; \cref{tab:sweep}).
As \cref{tab:visibility} shows, at $\epsilon=4/255$ all methods have a similar PSNR and SSIM at the native resolution, since all of them use the full budget.
At equal visibility (\cref{fig:tradeoff}), SRIM reaches 0.341 at an LPIPS of 0.10 against 0.159 for the best published method, and the gap is similar at equal PSNR.
In \cref{fig:visibility}, SRIM at a quarter of the budget is less visible (43 against 33\,dB) and still disrupts the edit far more (0.60 and 0.63 against 0.14 and 0.24).

\begin{figure}[!tb]
\centering
\includegraphics[width=\columnwidth]{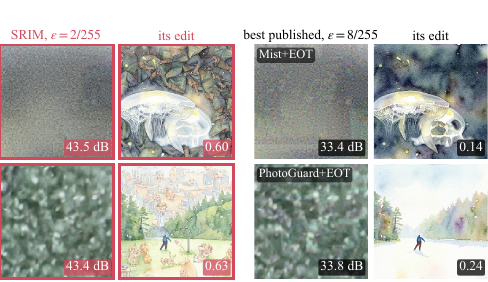}
\caption{Selected examples on FLUX.2-klein at $s=8$.
Each row shows a native-resolution $128\times128$ crop protected by SRIM at $\epsilon=2/255$ and by the best published protection at $\epsilon=8/255$, each followed by its edit.
Badges give the PSNR of the whole protected photo and the edit LPIPS to the unprotected edit.}
\label{fig:visibility}
\end{figure}

\begin{figure}[!tb]
\centering
\includegraphics[width=\columnwidth]{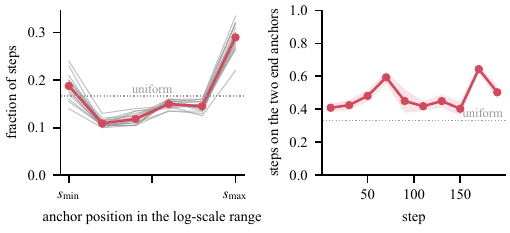}
\caption{Where SRIM spends its steps (16 development photos, six anchors each, FLUX.2 VAE). Left: fraction of the 200 steps drawn at each anchor against its position in the log-scale range (gray: photos; red: mean; dotted: the uniform share $1/K$). Right: share of the two end anchors in windows of 20 steps, mean and standard deviation over photos.}
\label{fig:anchors}
\end{figure}

\paragraph{Resampling filters and purification}
As shown in \cref{tab:robust}, the protection is insensitive to the editor's resampling filter.
Under purification, blurring and a round trip through the surrogate VAE leave SRIM largely intact (0.284 and 0.400) while reducing the published protections to at most 0.143.
Compression removes more (0.053 to 0.077 with noise of $8/255$), but SRIM remains the strongest in every setting.

\paragraph{Where the steps go}
\Cref{fig:anchors} records which anchor each step of SRIM draws on the 16 development photos, six anchors each.
The two end anchors receive 48\% of the 200 steps against 33\% under uniform sampling, the largest scale alone 29\%, and their share rises from 42\% in the first 40 steps to 57\% in the last 40 as the middle anchors stop falling behind.
The final deviation differs by only 0.09 across the six anchors, with the two ends lowest.
Weakness tracking therefore spends the budget where \cref{thm:eot} says that averaging gives up, at the ends of the range.
\paragraph{Temperature}
Varying $\tau$ over two orders of magnitude changes the worst case by at most 0.011, and even the near-uniform $\tau=10$ stays well above EOT (\cref{tab:tau}), so most of the gain comes from the anchor grid, which weights the ends of the range like its middle.
\paragraph{Cost}
\Cref{tab:cost} gives the cost of each protection on one L40S GPU for a test photo of 17 megapixels on average.
SRIM encodes at most a 5.1-megapixel copy per step and takes 176\,s, against 134 to 155\,s for the EOT variants and 59\,s at one megapixel.

\begin{table}[!tb]
\centering
\caption{Effect of the temperature $\tau$ of the anchor distribution (FLUX.2-klein, 40 test photos, $s\in\{2,4,8\}$).}
\label{tab:tau}
\setlength{\tabcolsep}{3mm}
\begin{tabular}{cccccc}
\toprule
$\tau$ & 0.03 & 0.1 (default) & 0.3 & 10 & EOT \\
\midrule
Avg. & 0.428 & 0.429 & 0.428 & 0.424 & 0.408 \\
Worst & 0.395 & 0.398 & 0.402 & 0.406 & 0.360\\
\bottomrule
\end{tabular}
\end{table}

\begin{table}[!tb]
\centering
\caption{Cost of protecting one test photo (17 megapixels on average, 200 steps) on one L40S GPU: the largest copy that the VAE encoder sees in a step and the mean time per photo. Native-resolution application exceeds the memory of the 48\,GB GPU for most test photos; its time is for a 12.6-megapixel photo that fit.}
\label{tab:cost}
\resizebox{\columnwidth}{!}{\begin{tabular}{lcc}
\toprule
Protection & Copy encoded per step & Time per photo \\
\midrule
PhotoGuard, Mist, PCA, Ours-obj.\ @1MP & 1.0 MP & 59 s \\
PCA+EOT & $\le$4.2 MP & 134 s \\
PhotoGuard+EOT & $\le$4.2 MP & 149 s \\
Mist+EOT & $\le$4.2 MP & 150 s \\
Ours-obj.+EOT & $\le$4.2 MP & 155 s \\
SRIM & $\le$5.1 MP & 176 s \\
\midrule
Native resolution (any objective) & 12.6 MP & 57 min \\
\bottomrule
\end{tabular}}
\end{table}
\section{Conclusion}
\label{sec:conclusion}

We studied the protection of full-resolution photos against editors that first downscale them by an unknown factor.
This resize decides whether a protection works: computed at the native resolution, protection decays with the factor; computed at a fixed resolution, it covers only a window of scales; and averaged over scales by EOT, it leaves the ends of the range exposed.
We proposed SRIM, which places anchors over the scale range implied by the photo size and spends each step on the weakest one, as cheaply as EOT.
On popular editors, SRIM gives the strongest worst-case protection, up to 2.4 times the best published one, and outperforms published protections even with a quarter of their $\ell_\infty$ budget.
\nocite{yam2025myface,guo2026artwork,guo2024laaca,liu2026diffprotect,guo2025pca,lin2026priorasskill,qian2026tdcount,wang2026harmreduction,zhao2026exploring,liu2026agentsurvey,liu2026phase,zhao2026metacognitive,yang2026pointcloud,zhou2026posecompass,li2026fedkd,zhao2026p2p,ai2026hardlabel,duan2026gaitprotector,lei2026gazesync,li2026safetwinnet,li2026w2sdefense,liu2026graybox,guo2025thesis,hu2025syntactic,li2025guardfed,guo2025beyond,li2025fairseg,zhao2025roptester,lei2025transferable,li2025flsurvey,qian2025perspective,qian2025t2icount,lei2025instant,zhao2025backdoorsurvey,liu2024mmadpurify,zhao2024weaktostrong,guo2024signature,qian2024semisupervised,guo2023siamese}
\bibliographystyle{IEEEtranN}
\bibliography{refs}

@inproceedings{athalye2018eot,
  author       = {Athalye, Anish and
                  Engstrom, Logan and
                  Ilyas, Andrew and
                  Kwok, Kevin},
  title        = {Synthesizing Robust Adversarial Examples},
  booktitle    = {Proceedings of the 35th International Conference on Machine Learning},
  series       = {Proceedings of Machine Learning Research},
  volume       = {80},
  pages        = {284--293},
  publisher    = {PMLR},
  year         = {2018},
}

@inproceedings{brooks2023ip2p,
  author       = {Brooks, Tim and
                  Holynski, Aleksander and
                  Efros, Alexei A.},
  title        = {{InstructPix2Pix}: Learning to Follow Image Editing Instructions},
  booktitle    = {Proceedings of the {IEEE/CVF} Conference on Computer Vision and Pattern
                  Recognition ({CVPR})},
  pages        = {18392--18402},
  publisher    = {IEEE},
  year         = {2023},
  doi          = {10.1109/CVPR52729.2023.01764},
}

@inproceedings{choi2025diffusionguard,
  author       = {Choi, William June Suk and
                  Lee, Kyungmin and
                  Jeong, Jongheon and
                  Xie, Saining and
                  Shin, Jinwoo and
                  Lee, Kimin},
  title        = {{DiffusionGuard}: A Robust Defense Against Malicious Diffusion-based Image
                  Editing},
  booktitle    = {International Conference on Learning Representations ({ICLR})},
  volume       = {2025},
  pages        = {27134--27180},
  year         = {2025},
}

@inproceedings{guo2019lowfreq,
  author       = {Guo, Chuan and
                  Frank, Jared S. and
                  Weinberger, Kilian Q.},
  title        = {Low Frequency Adversarial Perturbation},
  booktitle    = {Proceedings of The 35th Uncertainty in Artificial Intelligence Conference},
  series       = {Proceedings of Machine Learning Research},
  volume       = {115},
  pages        = {1127--1137},
  publisher    = {PMLR},
  year         = {2020},
}

@article{guo2025pca,
  author       = {Guo, Zhongliang and
                  Lei, Chun Tong and
                  Fang, Lei and
                  Zhao, Shuai and
                  Qian, Yifei and
                  Lin, Jingyu and
                  Wang, Zeyu and
                  Chen, Cunjian and
                  Arandjelovi{\'c}, Ognjen and
                  Lau, Chun Pong},
  title        = {A Gray-Box Attack Against Latent Diffusion Model-Based Image Editing by Posterior Collapse},
  journal      = {IEEE Transactions on Information Forensics and Security},
  volume       = {20},
  pages        = {12918--12933},
  year         = {2025},
  doi          = {10.1109/TIFS.2025.3639957},
}

@article{labs2025kontext,
  title={Flux. 1 kontext: Flow matching for in-context image generation and editing in latent space},
  author={Labs, Black Forest and Batifol, Stephen and Blattmann, Andreas and Boesel, Frederic and Consul, Saksham and Diagne, Cyril and Dockhorn, Tim and English, Jack and English, Zion and Esser, Patrick and others},
  journal={arXiv preprint arXiv:2506.15742},
  year={2025}
}

@inproceedings{liang2023advdm,
  author       = {Liang, Chumeng and
                  Wu, Xiaoyu and
                  Hua, Yang and
                  Zhang, Jiaru and
                  Xue, Yiming and
                  Song, Tao and
                  Xue, Zhengui and
                  Ma, Ruhui and
                  Guan, Haibing},
  title        = {Adversarial Example Does Good: Preventing Painting Imitation from Diffusion Models via Adversarial Examples},
  booktitle    = {Proceedings of the 40th International Conference on Machine Learning},
  series       = {Proceedings of Machine Learning Research},
  volume       = {202},
  pages        = {20763--20786},
  publisher    = {PMLR},
  year         = {2023},
}

@article{liang2023mist,
  title={Mist: Towards improved adversarial examples for diffusion models},
  author={Liang, Chumeng and Wu, Xiaoyu},
  journal={arXiv preprint arXiv:2305.12683},
  year={2023}
}

@inproceedings{lipman2023flowmatching,
  author       = {Lipman, Yaron and
                  Chen, Ricky T. Q. and
                  Ben-Hamu, Heli and
                  Nickel, Maximilian and
                  Le, Matt},
  title        = {Flow Matching for Generative Modeling},
  booktitle    = {International Conference on Learning Representations ({ICLR})},
  year         = {2023},
}

@inproceedings{madry2018pgd,
  author       = {Madry, Aleksander and
                  Makelov, Aleksandar and
                  Schmidt, Ludwig and
                  Tsipras, Dimitris and
                  Vladu, Adrian},
  title        = {Towards Deep Learning Models Resistant to Adversarial Attacks},
  booktitle    = {International Conference on Learning Representations ({ICLR})},
  year         = {2018},
}

@inproceedings{nie2022diffpure,
  author       = {Nie, Weili and
                  Guo, Brandon and
                  Huang, Yujia and
                  Xiao, Chaowei and
                  Vahdat, Arash and
                  Anandkumar, Animashree},
  title        = {Diffusion Models for Adversarial Purification},
  booktitle    = {Proceedings of the 39th International Conference on Machine Learning},
  series       = {Proceedings of Machine Learning Research},
  volume       = {162},
  pages        = {16805--16827},
  publisher    = {PMLR},
  year         = {2022},
}

@inproceedings{peebles2023dit,
  author       = {William Peebles and
                  Saining Xie},
  title        = {Scalable Diffusion Models with Transformers},
  booktitle    = {{IEEE/CVF} International Conference on Computer Vision, {ICCV} 2023, Paris,
                  France, October 1-6, 2023},
  pages        = {4172--4182},
  publisher    = {{IEEE}},
  year         = {2023},
  doi          = {10.1109/ICCV51070.2023.00387},
  bibsource    = {dblp computer science bibliography, https://dblp.org}
}

@article{purifyonce2026,
  title={Purify Once, Edit Freely: Breaking Image Protections under Model Mismatch},
  author={Zhao, Qichen and Zhai, Shengfang and Bai, Xinjian and Shen, Qingni and Lin, Qiqi and Gao, Yansong and Wu, Zhonghai},
  journal={arXiv preprint arXiv:2603.13028},
  year={2026}
}

@inproceedings{quiring2020scaling,
  author       = {Quiring, Erwin and
                  Klein, David and
                  Arp, Daniel and
                  Johns, Martin and
                  Rieck, Konrad},
  title        = {Adversarial Preprocessing: Understanding and Preventing Image-Scaling Attacks in Machine Learning},
  booktitle    = {29th {USENIX} Security Symposium ({USENIX} Security 20)},
  pages        = {1363--1380},
  publisher    = {USENIX Association},
  year         = {2020},
}

@inproceedings{radford2021clip,
  author       = {Radford, Alec and
                  Kim, Jong Wook and
                  Hallacy, Chris and
                  Ramesh, Aditya and
                  Goh, Gabriel and
                  Agarwal, Sandhini and
                  Sastry, Girish and
                  Askell, Amanda and
                  Mishkin, Pamela and
                  Clark, Jack and
                  Krueger, Gretchen and
                  Sutskever, Ilya},
  title        = {Learning Transferable Visual Models From Natural Language Supervision},
  booktitle    = {Proceedings of the 38th International Conference on Machine Learning},
  series       = {Proceedings of Machine Learning Research},
  volume       = {139},
  pages        = {8748--8763},
  publisher    = {PMLR},
  year         = {2021},
}

@inproceedings{sagawa2020dro,
  author       = {Sagawa, Shiori and
                  Koh, Pang Wei and
                  Hashimoto, Tatsunori B. and
                  Liang, Percy},
  title        = {Distributionally Robust Neural Networks},
  booktitle    = {International Conference on Learning Representations ({ICLR})},
  year         = {2020},
}

@inproceedings{salman2023photoguard,
  author       = {Salman, Hadi and
                  Khaddaj, Alaa and
                  Leclerc, Guillaume and
                  Ilyas, Andrew and
                  Madry, Aleksander},
  title        = {Raising the Cost of Malicious {AI}-Powered Image Editing},
  booktitle    = {Proceedings of the 40th International Conference on Machine Learning},
  series       = {Proceedings of Machine Learning Research},
  volume       = {202},
  pages        = {29894--29918},
  publisher    = {PMLR},
  year         = {2023},
}

@inproceedings{shan2023glaze,
  author       = {Shan, Shawn and
                  Cryan, Jenna and
                  Wenger, Emily and
                  Zheng, Haitao and
                  Hanocka, Rana and
                  Zhao, Ben Y.},
  title        = {{Glaze}: Protecting Artists from Style Mimicry by Text-to-Image Models},
  booktitle    = {32nd {USENIX} Security Symposium ({USENIX} Security 23)},
  pages        = {2187--2204},
  publisher    = {USENIX Association},
  year         = {2023},
}

@article{shen2025decontext,
  title={DeContext as Defense: Safe Image Editing in Diffusion Transformers},
  author={Shen, Linghui and Cui, Mingyue and Yang, Xingyi},
  journal={arXiv preprint arXiv:2512.16625},
  year={2025}
}

@article{sun2026doprotective,
  title={Do Protective Perturbations Really Protect Portrait Privacy under Real-world Image Transformations?},
  author={Sun, Ruiqing and Yao, Xingshan and Wu, Zhijing and Lan, Tian and Cui, Chenhao and Zhao, Huiyang and Shi, Jialing and Yang, Chen and Mao, Xianling},
  journal={arXiv preprint arXiv:2604.23688},
  year={2026}
}

@inproceedings{vanle2023antidreambooth,
  author       = {Thanh Van Le and
                  Hao Phung and
                  Thuan Hoang Nguyen and
                  Quan Dao and
                  Ngoc N. Tran and
                  Anh Tuan Tran},
  title        = {{Anti-DreamBooth}: Protecting users from personalized text-to-image
                  synthesis},
  booktitle    = {{IEEE/CVF} International Conference on Computer Vision, {ICCV} 2023, Paris,
                  France, October 1-6, 2023},
  pages        = {2116--2127},
  publisher    = {{IEEE}},
  year         = {2023},
  doi          = {10.1109/ICCV51070.2023.00202},
  bibsource    = {dblp computer science bibliography, https://dblp.org}
}

@inproceedings{wang2025facelock,
  author       = {Wang, Hanhui and
                  Zhang, Yihua and
                  Bai, Ruizheng and
                  Zhao, Yue and
                  Liu, Sijia and
                  Tu, Zhengzhong},
  title        = {Edit Away and My Face Will not Stay: Personal Biometric Defense against Malicious Generative Editing},
  booktitle    = {Proceedings of the {IEEE/CVF} Conference on Computer Vision and Pattern
                  Recognition ({CVPR})},
  pages        = {23806--23816},
  publisher    = {IEEE},
  year         = {2025},
  doi          = {10.1109/CVPR52734.2025.02217},
}

@article{wu2025qwenimage,
  author       = {Wu, Chenfei and
                  Li, Jiahao and
                  Zhou, Jingren and
                  Lin, Junyang and
                  Gao, Kaiyuan and
                  Yan, Kun and
                  Yin, Sheng-ming and
                  Bai, Shuai and
                  Xu, Xiao and
                  Chen, Yilei and
                  Chen, Yuxiang and
                  Tang, Zecheng and
                  Zhang, Zekai and
                  Wang, Zhengyi and
                  Yang, An and
                  Yu, Bowen and
                  Cheng, Chen and
                  Liu, Dayiheng and
                  Li, Deqing and
                  Zhang, Hang and
                  Meng, Hao and
                  Wei, Hu and
                  Ni, Jingyuan and
                  Chen, Kai and
                  Cao, Kuan and
                  Peng, Liang and
                  Qu, Lin and
                  Wu, Minggang and
                  Wang, Peng and
                  Yu, Shuting and
                  Wen, Tingkun and
                  Feng, Wensen and
                  Xu, Xiaoxiao and
                  Wang, Yi and
                  Zhang, Yichang and
                  Zhu, Yongqiang and
                  Wu, Yujia and
                  Cai, Yuxuan and
                  Liu, Zenan},
  title        = {{Qwen-Image} Technical Report},
  journal      = {arXiv preprint arXiv:2508.02324},
  year         = {2025},
  eprint       = {2508.02324},
  archiveprefix = {arXiv},
  primaryclass = {cs.CV}
}

@inproceedings{xiao2019scaling,
  author       = {Xiao, Qixue and
                  Chen, Yufei and
                  Shen, Chao and
                  Chen, Yu and
                  Li, Kang},
  title        = {Seeing is Not Believing: Camouflage Attacks on Image Scaling Algorithms},
  booktitle    = {28th {USENIX} Security Symposium ({USENIX} Security 19)},
  pages        = {443--460},
  publisher    = {USENIX Association},
  year         = {2019},
}

@inproceedings{xue2024sds,
  author       = {Xue, Haotian and
                  Liang, Chumeng and
                  Wu, Xiaoyu and
                  Chen, Yongxin},
  title        = {Toward effective protection against diffusion-based mimicry through score
                  distillation},
  booktitle    = {International Conference on Learning Representations ({ICLR})},
  volume       = {2024},
  pages        = {25734--25758},
  year         = {2024},
}

@inproceedings{yang2025antipure,
  author       = {Wenkui Yang and
                  Jie Cao and
                  Junxian Duan and
                  Ran He},
  title        = {Towards Robust Defense Against Customization via Protective Perturbation
                  Resistant to Diffusion-based Purification},
  booktitle    = {{IEEE/CVF} International Conference on Computer Vision, {ICCV} 2025, Honolulu,
                  HI, USA, October 19-25, 2025},
  pages        = {19290--19300},
  publisher    = {{IEEE}},
  year         = {2025},
  doi          = {10.1109/ICCV51701.2025.01793},
  bibsource    = {dblp computer science bibliography, https://dblp.org}
}

@article{yao2025eolt,
  title={Towards Robust Protective Perturbation against DeepFake Face Swapping},
  author={Yao, Hengyang and Li, Lin and Sun, Ke and Qiu, Jianing and Chen, Huiping},
  journal={arXiv preprint arXiv:2512.07228},
  year={2025}
}

@inproceedings{zhang2018lpips,
  author       = {Zhang, Richard and
                  Isola, Phillip and
                  Efros, Alexei A. and
                  Shechtman, Eli and
                  Wang, Oliver},
  title        = {The Unreasonable Effectiveness of Deep Features as a Perceptual Metric},
  booktitle    = {Proceedings of the {IEEE/CVF} Conference on Computer Vision and Pattern
                  Recognition ({CVPR})},
  pages        = {586--595},
  publisher    = {IEEE},
  year         = {2018},
  doi          = {10.1109/CVPR.2018.00068},
}

@article{zhang2025transferable,
  author       = {Zhang, Jie and
                  Dong, Shuai and
                  Shan, Shiguang and
                  Chen, Xilin},
  title        = {Toward Transferable Defense Against Malicious Image Edits},
  journal      = {IEEE Transactions on Pattern Analysis and Machine Intelligence},
  volume       = {48},
  number       = {7},
  pages        = {8134--8147},
  year         = {2026},
  doi          = {10.1109/TPAMI.2026.3670292},
}

@article{zhang2026survey,
  title={Adversarial Attacks for Good: A Survey of Proactive Protection across the Visual Content Lifecycle},
  author={Zhang, Jiaming and Chen, Boyang and Li, Zherui and Zhang, Fuyao and Yan, Xinyu and Tae, Hong Xi and He, Wenwen and Wang, Xuan and Guo, Siqi and Dong, Junhao and others},
  journal={arXiv preprint arXiv:2608.04314},
  year={2026}
}

@article{liu2026diffprotect,
  author       = {Liu, Jiang and
                  Lau, Chun Pong and
                  Guo, Zhongliang and
                  Guo, Yuxiang and
                  Wang, Zhaoyang and
                  Chellappa, Rama},
  title        = {{DiffProtect}: Generative adversarial examples using diffusion models for facial privacy protection},
  journal      = {Pattern Recognition},
  volume       = {173},
  pages        = {112780},
  year         = {2026},
  doi          = {10.1016/j.patcog.2025.112780},
}

@article{guo2026artwork,
  author       = {Guo, Zhongliang and
                  Qian, Yifei and
                  Zhao, Shuai and
                  Dong, Junhao and
                  Li, Yanli and
                  Arandjelovi{\'c}, Ognjen and
                  Fang, Lei and
                  Lau, Chun Pong},
  title        = {Artwork protection against unauthorized neural style transfer and aesthetic color distance metric},
  journal      = {Pattern Recognition},
  volume       = {171},
  pages        = {112105},
  year         = {2026},
  doi          = {10.1016/j.patcog.2025.112105},
}

@inproceedings{guo2024laaca,
  author       = {Guo, Zhongliang and
                  Dong, Junhao and
                  Qian, Yifei and
                  Wang, Kaixuan and
                  Li, Weiye and
                  Guo, Ziheng and
                  Wang, Yuheng and
                  Li, Yanli and
                  Arandjelovi{\'c}, Ognjen and
                  Fang, Lei},
  title        = {Artwork Protection Against Neural Style Transfer Using Locally Adaptive Adversarial Color Attack},
  booktitle    = {European Conference on Artificial Intelligence ({ECAI})},
  series       = {Frontiers in Artificial Intelligence and Applications},
  volume       = {392},
  pages        = {1414--1421},
  publisher    = {IOS Press},
  year         = {2024},
  doi          = {10.3233/FAIA240643},
}

@article{lau2023ijsat,
  author       = {Lau, Chun Pong and
                  Liu, Jiang and
                  Souri, Hossein and
                  Lin, Wei-An and
                  Feizi, Soheil and
                  Chellappa, Rama},
  title        = {Interpolated Joint Space Adversarial Training for Robust and Generalizable Defenses},
  journal      = {IEEE Transactions on Pattern Analysis and Machine Intelligence},
  volume       = {45},
  number       = {11},
  pages        = {13054--13067},
  year         = {2023},
  doi          = {10.1109/TPAMI.2023.3286772},
}

@article{liu2022mutualat,
  author       = {Liu, Jiang and
                  Lau, Chun Pong and
                  Souri, Hossein and
                  Feizi, Soheil and
                  Chellappa, Rama},
  title        = {Mutual Adversarial Training: Learning Together is Better Than Going Alone},
  journal      = {IEEE Transactions on Information Forensics and Security},
  volume       = {17},
  pages        = {2364--2377},
  year         = {2022},
  doi          = {10.1109/TIFS.2022.3184262},
}

@inproceedings{rombach2022ldm,
  author       = {Rombach, Robin and
                  Blattmann, Andreas and
                  Lorenz, Dominik and
                  Esser, Patrick and
                  Ommer, Bj{\"o}rn},
  title        = {High-Resolution Image Synthesis with Latent Diffusion Models},
  booktitle    = {Proceedings of the {IEEE/CVF} Conference on Computer Vision and Pattern
                  Recognition ({CVPR})},
  pages        = {10674--10685},
  publisher    = {IEEE},
  year         = {2022},
  doi          = {10.1109/CVPR52688.2022.01042},
}

@article{wang2004ssim,
  author       = {Wang, Zhou and
                  Bovik, Alan C. and
                  Sheikh, Hamid R. and
                  Simoncelli, Eero P.},
  title        = {Image quality assessment: from error visibility to structural similarity},
  journal      = {IEEE Transactions on Image Processing},
  volume       = {13},
  number       = {4},
  pages        = {600--612},
  year         = {2004},
  doi          = {10.1109/TIP.2003.819861},
}

@article{yam2025myface,
  title={My face is mine, not yours: Facial protection against diffusion model face swapping},
  author={Yam, Hon Ming and Guo, Zhongliang and Lau, Chun Pong},
  journal={arXiv preprint arXiv:2505.15336},
  year={2025}
}

@article{ahn2026imperceptible,
  author       = {Ahn, Namhyuk and Ahn, Wonhyuk and Yoo, KiYoon and Kim, Daesik and Nam, Seung-Hun},
  title        = {Imperceptible Protection Against Style Imitation From Diffusion Models},
  journal      = {IEEE Transactions on Multimedia},
  volume       = {28},
  pages        = {4286--4300},
  year         = {2026},
  doi          = {10.1109/TMM.2026.3660109},
}

@article{zhu2025lowfreq,
  author       = {Zhu, Jiang and Tan, Lingping and Li, Yanchun and Tian, Shujuan and Li, Jianqi and Wang, Yaonan},
  title        = {Guided Adversarial Attack in the Low-Frequency Space},
  journal      = {IEEE Transactions on Multimedia},
  volume       = {27},
  pages        = {8029--8044},
  year         = {2025},
  doi          = {10.1109/TMM.2025.3604964},
}

@article{mustafa2020superres,
  author       = {Mustafa, Aamir and Khan, Salman H. and Hayat, Munawar and Shen, Jianbing and Shao, Ling},
  title        = {Image Super-Resolution as a Defense Against Adversarial Attacks},
  journal      = {IEEE Transactions on Image Processing},
  volume       = {29},
  pages        = {1711--1724},
  year         = {2020},
  doi          = {10.1109/TIP.2019.2940533}
}

@article{weng2026frequencies,
  author       = {Weng, Juanjuan and Luo, Zhiming and Li, Shaozi},
  title        = {Exploring Frequencies via Feature Mixing and Meta-Learning for Improving Adversarial Transferability},
  journal      = {IEEE Transactions on Image Processing},
  volume       = {35},
  pages        = {1350--1363},
  year         = {2026},
  doi          = {10.1109/TIP.2026.3657166}
}

@article{huang2023erosion,
  author       = {Huang, Lifeng and Gao, Chengying and Liu, Ning},
  title        = {Erosion Attack: Harnessing Corruption to Improve Adversarial Examples},
  journal      = {IEEE Transactions on Image Processing},
  volume       = {32},
  pages        = {4828--4841},
  year         = {2023},
  doi          = {10.1109/TIP.2023.3251719}
}

@article{wang2024uap,
  author       = {Wang, Donghua and Yao, Wen and Jiang, Tingsong and Chen, Xiaoqian},
  title        = {Improving Transferability of Universal Adversarial Perturbation With Feature Disruption},
  journal      = {IEEE Transactions on Image Processing},
  volume       = {33},
  pages        = {722--737},
  year         = {2024},
  doi          = {10.1109/TIP.2023.3345136}
}

@article{jia2022fastat,
  author       = {Jia, Xiaojun and Zhang, Yong and Wu, Baoyuan and Wang, Jue and Cao, Xiaochun},
  title        = {Boosting Fast Adversarial Training With Learnable Adversarial Initialization},
  journal      = {IEEE Transactions on Image Processing},
  volume       = {31},
  pages        = {4417--4430},
  year         = {2022},
  doi          = {10.1109/TIP.2022.3184255}
}

@article{lin2026priorasskill,
  author    = {Lin, Yan and Sun, Jingyu and Qian, Yifei and Yao, Zhengtao and Wang, Yuxia and Liu, Daizong and Lau, Chun Pong and Guo, Zhongliang and McGough, Stephen},
  title     = {Prior-as-Skill: Remember Outside and Read Inside with On-Policy Self-Distillation for Self-Evolving {VLM} Agents},
  journal   = {arXiv preprint arXiv:2611.15766},
  year      = {2026},
}

@article{qian2026tdcount,
  author    = {Qian, Yifei and Guo, Zhongliang and Lin, Yan and Lei, Chun Tong and Gao, Li and Deng, Bowen and Liu, Daizong and Lau, Chun Pong and Hong, Xiaopeng and Pound, Michael P.},
  title     = {{TDCount}: Thermal-only crowd counting for privacy-conscious deployment},
  journal   = {Pattern Recognition},
  volume    = {183},
  pages     = {115129},
  year      = {2027}
}

@inproceedings{wang2026harmreduction,
  author    = {Wang, Kaixuan and Diao, Chenxin and Jacques, Jason T. and Guo, Zhongliang and Clarke, Loraine and Zhao, Shuai},
  title     = {{HarmReduction}: Benchmarking {LLMs} in Harm Reduction Information Provision to Support People Who Use Drugs},
  booktitle = {Proceedings of the 2026 Conference on Empirical Methods in Natural Language Processing ({EMNLP})},
  pages     = {6382--6400},
  year      = {2026},
}

@article{zhao2026exploring,
  author    = {Zhao, Shuai and Guo, Zhongliang and Xiao, Luwei and Nguyen, Cong-Duy and Jia, Yanhao and Wu, Xiaobao and Luu, Anh Tuan},
  title     = {Exploring backdoor attack and defense algorithms in {LLMs}: Enhancing in-context learning security},
  journal   = {Pattern Recognition},
  volume    = {183},
  pages     = {115005},
  year      = {2027}
}

@article{liu2026agentsurvey,
  author    = {Liu, Daizong and Luo, Tianyao and Huang, Shuwei and Yang, Junyi and Liang, Siyuan and Jia, Xiaojun and Lau, Chun Pong and Dong, Jianfeng and Liu, Yan and Bai, Yang and Chen, Xun and Zhou, Shiji and Arandjelovi{\'c}, Ognjen and Koffas, Stefanos and Guo, Zhongliang and Guan, Runwei and Zhou, Pan and Wang, Zhibo and Qin, Zhan and Du, Bo and Xiong, Hui and Chen, Pin-Yu and Domingo-Ferrer, Josep and Li, Yiming},
  title     = {A survey on {AI} agent security: Reasoning, Acting, and Self-Evolving},
  journal   = {AI Plus},
  volume    = {1},
  number    = {1},
  pages     = {100004-122--100004-143},
  year      = {2026},
}

@inproceedings{liu2026phase,
  author    = {Liu, Daizong and Dong, Junhao and Fang, Xiang and He, Hongyang and Jin, Keyan and Guo, Zhongliang and Qu, Xiaoye and Tang, Keke},
  title     = {Understanding and Exploiting Phase Sensitivity for Attacking Large Vision--Language Models},
  booktitle = {Proceedings of the International Joint Conference on Artificial Intelligence ({IJCAI})},
  pages     = {1424--1432},
  year      = {2026}
}

@article{zhao2026metacognitive,
  author    = {Zhao, Shuai and Guo, Zhongliang and Wu, Xinyi and Wu, Xiaobao and Jia, Yanhao and Xiao, Luwei and Luu, Anh Tuan},
  title     = {Protecting Your Customized {LLM} Systems from Backdoored Instructions with Metacognitive Probing},
  journal   = {IEEE Transactions on Information Forensics and Security},
  volume    = {21},
  pages     = {7110--7124},
  year      = {2026}
}

@article{yang2026pointcloud,
  author    = {Yang, Li and Cai, Xiaowen and Chen, Shuqin and Dong, Junhao and Tang, Keke and Guo, Zhongliang and Liu, Daizong},
  title     = {Rethinking {3D} Point Cloud Adversarial Attacks From Models' Inherent Focus},
  journal   = {Expert Systems with Applications},
  volume    = {332},
  pages     = {133660},
  year      = {2027}
}

@inproceedings{zhou2026posecompass,
  author    = {Zhou, Yanan and Qian, Zhaoyan and Li, Yanli and Yang, Nan and Guo, Zhongliang and Yuan, Dong},
  title     = {{PoseCompass}: Intelligent Synthetic Pose Selection for Visual Localization},
  booktitle = {Proceedings of the IEEE International Conference on Multimedia and Expo ({ICME})},
  year      = {2026},
}

@inproceedings{li2026fedkd,
  author    = {Li, Yuqi and Lin, Xingyou and Li, Yanli and Zhang, Kai and Yang, Chuanguang and Guo, Zhongliang and Gou, Jianping and Huang, Tingwen and Tian, Yingli},
  title     = {Federated Knowledge Distillation for Multi-Model Architectures Lithography Hotspot Detection},
  booktitle = {Proceedings of the IEEE International Conference on Multimedia and Expo ({ICME})},
  year      = {2026},
}

@inproceedings{zhao2026p2p,
  author    = {Zhao, Shuai and Wu, Xinyi and Zhao, Shiqian and Wu, Xiaobao and Guo, Zhongliang and Jia, Yanhao and Luu, Anh Tuan},
  title     = {{P2P}: A poison-to-poison remedy for reliable backdoor defense in {LLMs}},
  booktitle = {Findings of the Association for Computational Linguistics: {ACL} 2026},
  pages     = {12345--12360},
  year      = {2026}
}

@article{ai2026hardlabel,
  author    = {Ai, Nian and Chen, Guangke and Cai, Xiaowen and Guo, Zhongliang and Liu, Daizong and Zhou, Pan and Arandjelovi{\'c}, Ognjen},
  title     = {Attacking hard-label large vision--language models with model-sensitive adversarial patch designs},
  journal   = {Pattern Recognition},
  volume    = {180},
  pages     = {114261},
  year      = {2026}
}

@inproceedings{duan2026gaitprotector,
  author    = {Duan, Huiran and Zhou, Qian and Guo, Zhongliang and Dong, Junhao and Li, Yuqi and Zhao, Guoying and Tian, Yingli},
  title     = {{GaitProtector}: Impersonation-Driven Gait De-Identification via Training-Free Diffusion Latent Optimization},
  booktitle = {Proceedings of the IEEE 20th International Conference on Automatic Face and Gesture Recognition ({FG})},
  pages     = {1--10},
  year      = {2026}
}

@inproceedings{lei2026gazesync,
  author    = {Lei, Yaxiong and Talwar, Rishab and He, Shijing and Gong, Xinya and Wang, Yuheng and Cai, Xudong and Guo, Zhongliang and Ye, Juan},
  title     = {{GazeSync}: A Mobile Eye-Tracking Tool for Analyzing Visual Attention on Dynamically Manipulated Content},
  booktitle = {Extended Abstracts of the 2026 {CHI} Conference on Human Factors in Computing Systems},
  year      = {2026}
}

@article{li2026safetwinnet,
  author    = {Li, Yanli and Zhang, Chenghao and Guo, Zhongliang and Yang, Nan and Yuan, Dong and Ding, Weiping},
  title     = {{SafeTwinNet}: Secure and Adaptive Learning Ecosystems for Digital Twins},
  journal   = {IEEE Communications Standards Magazine},
  year      = {2026}
}

@article{li2026w2sdefense,
  author    = {Li, Yuwen and Wu, Xinyi and Guo, Zhongliang and Xiao, Luwei and Jia, Yanhao and Zhao, Shuai},
  title     = {Backdoor Defense for Large Language Models with Weak-to-Strong Knowledge Distillation},
  journal   = {Pattern Recognition},
  volume    = {175},
  pages     = {113030},
  year      = {2026}
}

@inproceedings{liu2026graybox,
  author    = {Liu, Daizong and Cai, Xiaowen and Dong, Junhao and Guo, Zhongliang and Qu, Xiaoye and Guan, Runwei and Fang, Xiang and Ye, Dengpan},
  title     = {Attacking gray-box large vision-language models with adaptive {SVD}-structured adversarial alignment},
  booktitle = {Proceedings of the International Conference on Machine Learning ({ICML})},
  year      = {2026},
}

@phdthesis{guo2025thesis,
  author    = {Guo, Zhongliang},
  title     = {Building Trustworthy Computer Vision: Adversarial Techniques for Robustness Assessment and Misuse Prevention},
  school    = {The University of St Andrews},
  year      = {2025},
}

@article{hu2025syntactic,
  author    = {Hu, Man and Yang, Yatao and Pan, Deng and Guo, Zhongliang and Xiao, Luwei and Lin, Deyu and Zhao, Shuai},
  title     = {Syntactic paraphrase-based synthetic data generation for backdoor attacks against {Chinese} language models},
  journal   = {Information Fusion},
  volume    = {124},
  pages     = {103376},
  year      = {2025},
}

@article{li2025guardfed,
  author    = {Li, Yanli and Zhou, Yanan and Guo, Zhongliang and Yang, Nan and Zhang, Yuning and Chen, Huaming and Yuan, Dong and Ding, Weiping and Pedrycz, Witold},
  title     = {{GuardFed}: A Trustworthy Federated Learning Framework Against Dual-Facet Attacks},
  journal   = {arXiv preprint arXiv:2511.09294},
  year      = {2025},
}

@article{guo2025beyond,
  author    = {Guo, Zhongliang and Qian, Yifei and Li, Yanli and Li, Weiye and Lei, Chun Tong and Zhao, Shuai and Fang, Lei and Arandjelovi{\'c}, Ognjen and Lau, Chun Pong},
  title     = {Beyond vulnerabilities: A survey of adversarial attacks as both threats and defenses in computer vision systems},
  journal   = {arXiv preprint arXiv:2508.01845},
  year      = {2025},
}

@article{li2025fairseg,
  author    = {Li, Yuqi and Li, Yanli and Zhang, Kai and Zhang, Fuyuan and Yang, Chuanguang and Guo, Zhongliang and Ding, Weiping and Huang, Tingwen},
  title     = {Achieving fair medical image segmentation in foundation models with adversarial visual prompt tuning},
  journal   = {Information Sciences},
  volume    = {720},
  pages     = {122501},
  year      = {2025}
}

@article{zhao2025roptester,
  author    = {Zhao, Shuai and Zhang, Yulin and Xiao, Luwei and Wu, Xinyi and Jia, Yanhao and Guo, Zhongliang and Wu, Xiaobao and Nguyen, Cong-Duy and Zhang, Guoming and Luu, Anh Tuan},
  title     = {{Affective-ROPTester}: Capability and Bias Analysis of {LLMs} in Predicting Retinopathy of Prematurity},
  journal   = {IEEE Transactions on Affective Computing},
  year      = {2025}
}

@article{lei2025transferable,
  author    = {Lei, Chun Tong and Guo, Zhongliang and Lee, Hon Chung and Duong, Minh Quoc and Lau, Chun Pong},
  title     = {Towards more transferable adversarial attack in black-box manner},
  journal   = {arXiv preprint arXiv:2505.18097},
  year      = {2025},
}

@article{li2025flsurvey,
  author    = {Li, Yanli and Guo, Zhongliang and Yang, Nan and Chen, Huaming and Yuan, Dong and Ding, Weiping},
  title     = {Threats and Defenses in the Federated Learning Life Cycle: A Comprehensive Survey and Challenges},
  journal   = {IEEE Transactions on Neural Networks and Learning Systems},
  volume    = {36},
  number    = {9},
  pages     = {15643--15663},
  year      = {2025}
}

@article{qian2025perspective,
  author    = {Qian, Yifei and Zhang, Liangfei and Guo, Zhongliang and Hong, Xiaopeng and Arandjelovi{\'c}, Ognjen and Donovan, Carl R.},
  title     = {Perspective-assisted prototype-based learning for semi-supervised crowd counting},
  journal   = {Pattern Recognition},
  volume    = {158},
  pages     = {111073},
  year      = {2025}
}

@inproceedings{qian2025t2icount,
  author    = {Qian, Yifei and Guo, Zhongliang and Deng, Bowen and Lei, Chun Tong and Zhao, Shuai and Lau, Chun Pong and Hong, Xiaopeng and Pound, Michael P.},
  title     = {{T2ICount}: Enhancing Cross-modal Understanding for Zero-Shot Counting},
  booktitle = {Proceedings of the {IEEE/CVF} Conference on Computer Vision and Pattern Recognition ({CVPR})},
  pages     = {25336--25345},
  year      = {2025},
}

@inproceedings{lei2025instant,
  author    = {Lei, Chun Tong and Yam, Hon Ming and Guo, Zhongliang and Qian, Yifei and Lau, Chun Pong},
  title     = {Instant adversarial purification with adversarial consistency distillation},
  booktitle = {Proceedings of the {IEEE/CVF} Conference on Computer Vision and Pattern Recognition ({CVPR})},
  pages     = {24331--24340},
  year      = {2025},
}

@article{zhao2025backdoorsurvey,
  author    = {Zhao, Shuai and Jia, Meihuizi and Guo, Zhongliang and Gan, Leilei and Xu, Xiaoyu and Wu, Xiaobao and Fu, Jie and Feng, Yichao and Pan, Fengjun and Luu, Anh Tuan},
  title     = {A Survey of Recent Backdoor Attacks and Defenses in Large Language Models},
  journal   = {Transactions on Machine Learning Research},
  year      = {2025}
}

@article{liu2024mmadpurify,
  author    = {Liu, Xinxin and Guo, Zhongliang and Huang, Siyuan and Lau, Chun Pong},
  title     = {{MMAD-Purify}: A precision-optimized framework for efficient and scalable multi-modal attacks},
  journal   = {arXiv preprint arXiv:2410.14089},
  year      = {2024},
}

@article{zhao2024weaktostrong,
  author    = {Zhao, Shuai and Gan, Leilei and Guo, Zhongliang and Wu, Xiaobao and Xiao, Luwei and Xu, Xiaoyu and Nguyen, Cong-Duy and Luu, Anh Tuan},
  title     = {Weak-to-strong backdoor attack for large language models},
  journal   = {arXiv preprint arXiv:2409.17946},
  year      = {2024},
}

@inproceedings{guo2024signature,
  author    = {Guo, Zhongliang and Li, Weiye and Qian, Yifei and Arandjelovi{\'c}, Ognjen and Fang, Lei},
  title     = {A White-Box False Positive Adversarial Attack Method on Contrastive Loss Based Offline Handwritten Signature Verification Models},
  booktitle = {Proceedings of the International Conference on Artificial Intelligence and Statistics ({AISTATS})},
  pages     = {901--909},
  year      = {2024},
}

@article{qian2024semisupervised,
  author    = {Qian, Yifei and Hong, Xiaopeng and Guo, Zhongliang and Arandjelovi{\'c}, Ognjen and Donovan, Carl R.},
  title     = {Semi-Supervised Crowd Counting with Contextual Modeling: Facilitating Holistic Understanding of Crowd Scenes},
  journal   = {IEEE Transactions on Circuits and Systems for Video Technology},
  volume    = {34},
  number    = {9},
  pages     = {8230--8241},
  year      = {2024}
}

@article{guo2023siamese,
  author    = {Guo, Zhongliang and Arandjelovi{\'c}, Ognjen and Reid, David and Lei, Yaxiong and B{\"u}ttner, Jochen},
  title     = {A {Siamese} Transformer Network for Zero-Shot Ancient Coin Classification},
  journal   = {Journal of Imaging},
  volume    = {9},
  number    = {6},
  pages     = {107},
  year      = {2023},
}
\end{document}